\documentclass[hidelinks, 11pt]{article}
\usepackage{geometry}
\usepackage{setspace}
\usepackage{bm}
\usepackage{graphicx}
\graphicspath{ {.} }
\usepackage{hyperref}
\usepackage{xr}

\usepackage{amssymb, amsmath, amsthm, amsfonts, mathtools, bbm}
\usepackage{float}
\usepackage{color, xcolor,colortbl}
\usepackage[dvipsnames]{xcolor}
\usepackage{url}
\usepackage{enumitem}
\usepackage{multirow}
\usepackage{siunitx}
\usepackage{placeins}
\usepackage{authblk}
\usepackage{appendix}
\usepackage{lscape}
\usepackage{booktabs}
\usepackage{longtable}
\usepackage{stackrel}

\usepackage{comment}

\usepackage{soul}       
\usepackage{ulem}   
\usepackage{nicefrac}
\usepackage{microtype}

\allowdisplaybreaks[1]

\usepackage{xr-hyper}
\usepackage{hyperref}
\hypersetup{
    colorlinks=true,
    linkcolor=blue,
    citecolor=blue,
    filecolor=magenta,
    urlcolor=cyan
}
\usepackage{cleveref}

\usepackage[round]{natbib}

\usepackage{amsthm} 

\newtheorem{theorem}{Theorem}
\newtheorem{proposition}{Proposition}
\newtheorem{lemma}{Lemma}

\newtheorem{assumption}{Assumption}

\usepackage{algorithm}
\usepackage{algorithmic}

\usepackage{bm}
\usepackage{amsmath}
\usepackage{multirow}
\usepackage{booktabs}
\usepackage[table]{xcolor}
\usepackage{tabularx}
\usepackage{array}
\usepackage[T1]{fontenc}
\usepackage[utf8]{inputenc}
\usepackage[ngerman,english]{babel}
\usepackage{subcaption}

\usepackage[makeroom]{cancel}

\usepackage{parskip}

\begin{document}

	\begin{titlepage}
		
		\title{Learning Perturbations to Extrapolate Your LLM}
                \author[1]{Zetai Cen\thanks{Equal contribution. 
                Email: zetai.cen@bristol.ac.uk.
                Supported by Engineering and Physical Sciences Research Council (EP/Z531327/1).}}
                \author[2]{Chenfei Gu\thanks{Equal contribution.
                Email: gu.chenfei@live.sufe.edu.cn}}
                \author[3]{Jin Zhu\thanks{
                Email: j.zhu.7@bham.ac.uk}}
                \author[2]{Ting Li\thanks{
                Email: tingli@mail.shufe.edu.cn}}
                \author[4]{Yunxiao Chen\thanks{Co-corresponding author. 
                Email: y.chen186@lse.ac.uk}}
                \author[4]{Chengchun Shi\thanks{Co-corresponding author. 
                Email: c.shi7@lse.ac.uk}}
		
        \affil[1]{School of Mathematics, University of Bristol}
        \affil[2]{School of Statistics and Data Science, Shanghai University of Finance and Economics}
        \affil[3]{School of Mathematics, University of Birmingham}
        \affil[4]{Department of Statistics, London School of Economics and Political Science}
		
		\date{}
		
		\maketitle

\begin{abstract}
Recent advancements in large language models demonstrate that injecting perturbations can substantially enhance extrapolation performance. However, current approaches often rely on discrete perturbations with fixed designs, which limits their flexibility. In this work, we propose a framework where token prefixes are perturbed by a learnable transformation of a continuous latent vector within an embedding space. To overcome the challenge of an intractable marginal likelihood, we derive unbiased estimating equations for model parameters and optimize them via stochastic gradient descent. We establish the statistical properties of the resulting estimator in over-parameterized regimes. Empirical evaluations on both synthetic and real-world datasets demonstrate that our proposal yields significant gains in out-of-domain settings over a range of state-of-the-art baseline methods.
\end{abstract}
		
		\bigskip
		\bigskip


\noindent

	\end{titlepage}
	
	\setcounter{page}{2}

\maketitle


\newpage

\section{Introduction}\label{sec:intro}

Training large language models (LLMs) such as GPT-5 and Qwen-3 \citep{openai2025gpt, yang2025qwen3} on massive text corpora aims at capturing the underlying distribution of natural language. 
Yet, it remains challenging for the trained model to extrapolate to out-of-distribution or out-of-domain settings beyond the support of its training data. 
The literature has seen the development of various data perturbation techniques, 
such as synonym replacement, random insertion, deletion, and swap, that modify training instances into semantically similar variants to effectively expose LLMs to a broader range of inputs and improve their ability to generalize beyond the training data \citep{feng2019keep, feng2020genaug, Lietal2024, Cenetal2026}.
However, their approach remains grounded in the discrete, word-level augmentation procedures mentioned previously, which may restrict its adaptivity across diverse domains.
While discrete perturbations are simple to use, they could be too coarse and hard to refine due to the complexity of natural language \citep{Parketal2022, Lietal2023_perturbscore}.
Meanwhile, fixed perturbations apply the same transformations to the data regardless of the contexts, thus failing to generalize appropriately \citep{IsmailovAsanova2025}.

This paper fills this gap. Our contribution is threefold: 

\begin{itemize}[leftmargin=0.5cm]
 \item[(i)] A new continuous perturbation framework is proposed for training and prediction of LLMs, which perturbs token prefixes by a trainable transformation of a low-dimensional latent vector within an embedding space. This framework learns the perturbation model jointly with the LLM, enabling adaptivity across tokens and contexts. Extensive experiments show that the proposed framework achieves more stable and effective extrapolation performance across different datasets and base language models. 
 \item[(ii)] A new training procedure is introduced, which tackles the challenge brought by the latency of the perturbation. That is, the marginal likelihood of an input prefix is intractable due to complex integrals over latent variables that determine the perturbation, making LLM training based on the likelihood function computationally infeasible. The proposed framework addresses the challenge by deriving an unbiased estimating function whose expectation vanishes at the true parameter and
 solving the resulting estimating equation by stochastic gradient descent. 
\item[(iii)] Theoretical guarantees are established for the trained model by establishing the statistical consistency of the parameter estimator
(Theorem~\ref{thm: consistency_2}) and deriving its rate of convergence (Theorem~\ref{thm: consistency_2_rate}). These results hold in an over-parametrized setting, where the true parameters are represented by a non-empty set with possibly more than one element.
\end{itemize}

\subsection{Related Works}\label{sec:related}

Our work is closely related to: training in LLMs, perturbations in natural language processing (NLP), and score-based approaches. We next discuss the related works in these areas.

\paragraph{Training in LLMs.}
Training LLMs typically consists of two stages. The first is a pre-training stage where the model is trained through autoregressive next‑token prediction that learns the conditional distribution of each token given its preceding ones.
This is implemented by minimizing the Kullback--Leibler (KL) divergence to the empirical data distribution \citep{radford2018improving, black2021gpt, zhang2022opt}.
Different training objectives have been proposed to improve the model robustness, such as
the total variation distance \citep{ji2023tailoring}, Earth Mover’s distance \citep{ren2024emo}, and mixed forward-reverse cross-entropies \citep{zhang2023mixce}.

The second stage is post-training that aligns the model's behavior with complex human values \citep{Ouyangetal2022,rafailov2023direct,liu2024dual,xiao2025algorithmic,xu2026doubly,xia2026statistical} and guides the model to address complex problems through step-by-step reasoning  \citep{shao2024deepseekmath,gong2026kernelized, huang2026learning,xie2026statistical}.

This paper focuses on pre-training. 

\paragraph{Perturbations in NLP.}
Perturbing language inputs has become a common strategy in NLP.
Discrete perturbations can be sentence-level \citep{Iyyeretal2018, Ribeiroetal2018}, word-level \citep{Renetal2019generating, wei2019eda, Zangetal2020word, Cenetal2026}, and character-level \citep{pruthi2019combating}, to name but a few types.
These operations are mostly heuristic and agnostic to the data and the model, which limits their use in practice.
Recent works have partially addressed this by introducing model-aware discrete perturbations \citep{Parketal2022}. 
This remains restricted in a discrete token space, making it difficult to impose smooth perturbations that can generalize across contexts.
On the other hand, continuous perturbations are often random noises with a hyperparameter controlling their range \citep{Zhu2020FreeLB, Lietal2023_perturbscore},
which remains ad-hoc.
This work directly builds on this, in that we integrate any base LLM with a parameterized, continuous perturbation process in the embedding space which is jointly learned with the LLM parameters.



\paragraph{Score-based approaches in machine learning.}
Score-based approaches have been used as a flexible alternative to maximum likelihood, particularly when the full likelihood is intractable \citep{LiangZeger1986}.
In modern machine learning, they have been adapted to high-dimensional and transfer learning settings \citep{Songetal2024, YanChen2024}.
Our construction is close to score‑based methods that subtract an expectation under a model to obtain unbiased estimating functions \citep{Hyvarinen2005, GutmannHyvarinen2010}, which have been extended to diffusion models and LLM alignment \citep{Songetal2021, Chenetal2024, Louetal2024}.
In our framework, the perturbation process plays a similar role, making the estimation procedure both computationally scalable and statistically principled.
For optimization, we directly apply stochastic gradient descent (SGD) to solve the empirical loss derived from the estimating equation, without variational approximations. 

\section{Methodology}\label{sec:method}

We develop an autoregressive language model with a trainable perturbation process. 
Text tokens are represented by some embedding vectors with ambient dimension $d$.
We assume any input sequence of length $L$, and shorter inputs are extended to this length by zero padding.
The perturbation and next-token prediction proceed as follows.

\textbf{Model}.
We assume the access of some fixed functions that embed any input token $O$ into a vector $X \in \mathbb{R}^d$, and a decoder that converts $X$ back to $O$.
An input prefix $X_{<t} \in \mathbb{R}^{d \times (t-1)}$ is perturbed by adding a matrix $W_{<t} = T_{\beta}(w|X_{<t})$, where $w \in \mathbb{R}^r$ is a low‑dimensional latent variable drawn from a known distribution (e.g.\ standard Gaussian) and $T_{\beta}(\bullet | X_{<t}): \mathbb{R}^r \mapsto \mathbb{R}^{d\times (t-1)}$ is a neural network parameterized by $\beta$. The perturbed prefix $\widetilde{X}_{<t} = X_{<t} + W_{<t}$ is then passed through a base autoregressive model $P_{\theta}$, parameterized by $\theta$, for predicting the next token embedding $X_{t}$.

Let $p$ be a known density function of $w$, which we set as standard Gaussian in experiments later.
Note that this does not restrict the flexibility of our perturbation process, as the perturbation process will be learned so that the distribution of $W_{<t}$ is flexible.
We denote all trainable parameters as $\gamma=(\theta, \beta)$. 
LLMs are typically overparametrized, so there might not exist a unique minimizer, and it is not identifiable. Hence, we consider a non-empty set of true parameters $\Gamma^*$ which will be detailed in Section~\ref{sec:theory}.
Denote $\Gamma$ as the set of trainable parameters such that $\Gamma^* \subseteq \Gamma$.
Consequently, each $X_t$ is drawn from the perturbed autoregressive model (conditional on the context):
\begin{equation*}
X_t \sim P_{\theta_0} \left( \bullet | X_{<t} +  T_{\beta_0}(w_t^* |X_{<t})  \right) ,
\end{equation*}
where $(\theta_0, \beta_0) \in \Gamma^*$,
and $w_t^* \sim p$ represents the unobserved latent variable that realizes $X_t$.

\begin{algorithm}[ht!]
\caption{Continuous Perturbation Autoregressive Training} 
\begin{algorithmic}[1]\label{alg:enpretrain_pro}
\STATE {\bf Require:} 
Training corpus $\mathcal{D}=\{O_{i, \leq L}\}_{i=1}^n$, perturbation step $K$, latent dimension $r$,
configuration $\textit{debias} = 1$ by default.
\STATE {Compute token embeddings $X_{i,\leq L} \leftarrow \mathrm{Embed}(O_{i, \leq L})$.}
\STATE {\bf for} {$i=1,\ldots, n$ {\bf do}}
\STATE \quad {\bf for} {$t=1,\ldots, L$ {\bf do}}
\STATE \quad \quad {\bf for} {$k=1,\ldots, K$ {\bf do}}
\STATE \quad \quad \quad {Sample $w_{it}^{(k)}$ from a pre-specified density function $p$.} 
\STATE \quad \quad \quad {Construct the perturbation matrix $W_{i,<t}^{(k)}=T_{\beta}(w_{it}^{(k)}|X_{i,<t})$.}
\STATE \quad \quad \quad {Compute $\ell_{it}^{(k)}=\log P_{\theta}(X_{it}|X_{i, <t}+W_{i, <t}^{(k)})$.}
\STATE \quad \quad \quad {\bf if} $\textit{debias} =1$ {\bf then }
\STATE \quad \quad \quad \quad {Sample $\widetilde{w}_{it}^{(k)}$ from $p$.} 
\STATE \quad \quad \quad \quad {Construct $\widetilde{W}_{i,<t}^{(k)}=T_{\beta}(\widetilde{w}_{it}^{(k)}|X_{i,<t})$.}
\STATE \quad \quad \quad \quad {Sample $\widetilde{X}_{it}^{(k)}$ from $P_{\theta}(\bullet|X_{i, <t}+\widetilde{W}_{i, <t}^{(k)})$.}
\STATE \quad \quad \quad \quad {Compute $\widetilde{\ell}_{it}^{(k)}=\log P_{\theta}(\widetilde{X}_{it}^{(k)}|X_{i, <t}+W_{i, <t}^{(k)})$.}
\STATE \quad \quad \quad {\bf else}
\STATE \quad \quad \quad \quad {Set $\widetilde{\ell}_{it}^{(k)}=0$.}
\STATE \quad \quad \quad {\bf end if}
\STATE \quad \quad {\bf end for}
\STATE \quad {\bf end for}
\STATE {\bf end for}
\STATE {Set 
    $\mathcal{L}(\gamma)=\sum_{i=1}^n \sum_{t=1}^{L} \sum_{k=1}^K \left(\ell_{it}^{(k)}-\widetilde{\ell}_{it}^{(k)}\right).$
}
\STATE {\bf Return:} {$\arg\max_{\gamma}\mathcal{L}(\gamma)$.}
\end{algorithmic}
\end{algorithm}

\begin{algorithm}[ht!]
\caption{Continuous Perturbation Autoregressive Inference} 
\begin{algorithmic}[1]\label{alg:enpretrain_pro_inference}
\STATE {\bf Require:} 
{Prompt $O_0$, perturbation module $T_{\beta}$, base language model $P_{\theta}$.}
\STATE \quad {\bf Initialize:} {Set $O_{<1}\leftarrow O_0$, $t\leftarrow 0$}
\STATE \quad {\bf repeat} 
\STATE \quad \quad {$t\leftarrow t+1$}
\STATE \quad \quad {Compute token embeddings
    \(X_{<t}\leftarrow\mathrm{Embed}(O_{<t})
    \).} 
\STATE \quad \quad {Sample latent perturbation $w_t \sim \mathcal{N}(0,I_r)$.}
\STATE \quad \quad {Generate perturbation matrix $W_{<t}\leftarrow T_\beta(w_t| X_{<t})$.}
\STATE \quad \quad {Construct perturbed embeddings $\widetilde X_{<t}\leftarrow X_{<t}+W_{<t}$.}
\STATE \quad \quad {Generate next token embedding $X_t\sim P_\theta(\bullet|\widetilde X_{<t})$.} 
\STATE \quad \quad {Generate next token $O_t$ from $X_t$.} 
\STATE \quad \quad {Append generated token $O_{<t+1}\leftarrow (O_{<t},O_t)$.}
\STATE \quad {\bf until} {$O_t$ is EOS}
\STATE {\bf Return:} {$O_{<t+1}$.}
\end{algorithmic}
\end{algorithm}

\textbf{Estimation}.
Algorithms~\ref{alg:enpretrain_pro} and \ref{alg:enpretrain_pro_inference} detail the training and inference procedures of our approach, respectively.
The high-level idea of our method is as follows. 
At each training step, we sample new latent variables and compare the score of the observed token with that of a synthetic token drawn from the model under the same perturbation.
This allows us to construct an objective function that can be optimized using SGD.
At the inference stage, the next-token prediction is repeatedly performed by sampling perturbation and feeding it to our trained model, until an end-of-sequence (EOS) token is generated.

The intuition of our method is that,
due to the latent perturbation process, the marginal likelihood of a sequence of embeddings is computationally infeasible to maximize. 
To resolve this and alleviate the computational burden, 
we may solve an unbiased score‑based estimating equation, 
which is in turn equivalent to maximizing the objective function $\mathcal{L}(\gamma)$ in Algorithm~\ref{alg:enpretrain_pro}.
All the technical nuances are elaborated in Appendix~\ref{subsec: additional_expla_method}.

\section{Theoretical Guarantees}\label{sec:theory}

In this section, we present the theoretical guarantee of the parameter estimators.
One restriction for this set of approach is that requiring the parameter to be uniquely identifiable would be unrealistic in our context with likely over-parameterization.
We relax the identification of model parameters through building our analysis upon a set of true parameters.

For any parameter $\gamma = (\theta, \beta)$, we define some notations as follows.
Given a sequence of prefixes $X_{\leq L}$, define the scalar quantity
\begin{equation*}
\begin{split}
    \ell(\gamma; X_{\leq L}) &:= 
    \sum_{k=1}^K \sum_{t=1}^L \Big\{ \log P_{\theta} \Big( X_t | X_{<t} +  T_\beta( w_t^{(k)} |X_{<t}) \Big) 
    - \log P_{\theta} \Big( \widetilde{X}_{t}^{(k)} | X_{<t} + T_\beta( w_t^{(k)} |X_{<t}) \Big) \Big\} .
\end{split}
\end{equation*}
and denote its gradient with respect to $\gamma$ as $\psi(\gamma; \bullet)$.
Define $J(\gamma) := \mathbb{E}[\ell(\gamma; \bullet)]$,
the gradient of $J$ as $\nabla J(\gamma) := \mathbb{E}[\psi(\gamma; \bullet)]$,
and the Hessian of $J$ as $H(\gamma) := \nabla^2 J(\gamma)$.
Denote $\Gamma$ the set of trainable parameters,
and $\Gamma^* := \{\gamma\in \Gamma: \nabla J(\gamma) =0\}$ the set of true parameters in the sense that they are the stationary points of $J$.
Let $\widehat{\gamma}$ be some sequence of parameter estimator
from Algorithm~\ref{alg:enpretrain_pro}.
We measure the distance between the parameter estimate and the ground truth as $d(\widehat{\gamma}, \Gamma^*) := \inf_{\gamma^*\in \Gamma^*} \|\widehat{\gamma} - \gamma^*\|$.

We impose the following assumptions, and then delineate the statistical properties of $\widehat{\gamma}$, with all the proofs deferred to the appendix.

\medskip
\begin{assumption}[Compactness]\label{ass: compact}
The set $\Gamma$ is compact,
and $\Gamma^*$ is in the interior of $\Gamma$.
\end{assumption}

\medskip
\begin{assumption}[Smoothness]\label{ass: smooth_J}
There exists an open neighborhood $\Gamma_0$ of $\Gamma^*$ such that
the function $\psi$ is continuously differentiable on $\Gamma_0$,
and denote the derivative at $\gamma = u$ as $\nabla_{\gamma} \psi(u; \bullet)$.
Further assume $\psi$ is square-integrable on $\Gamma$.
\end{assumption}


Assumptions~\ref{ass: compact}--\ref{ass: smooth_J} are standard regularity conditions in M‑estimation theory.
Markedly, we do not require a uniquely identifiable parameter, and our results are only upon the set $\Gamma^*$.

\medskip
\begin{theorem}[Asymptotic consistency]\label{thm: consistency_2}
Under Assumptions~\ref{ass: compact}--\ref{ass: smooth_J},
we have
$d(\widehat{\gamma}, \Gamma^*) \xrightarrow{p} 0$
as $n\to \infty$.
\end{theorem}

Theorem~\ref{thm: consistency_2} guarantees that our estimator from Algorithm~\ref{alg:enpretrain_pro} is asymptotically consistent, thus validating our method.
We present an additional assumption in the following, and reinforce the theoretical result by deriving the convergence rate.

\medskip
\begin{assumption}[Hessian matrices]\label{ass: Hessian}
(i) $\Gamma^*$ is closed and convex;
(ii) $H(\gamma^*)$ exists, is continuous, and has the same rank for all $\gamma^* \in \Gamma^*$;
(iii) There exists an orthonormal matrix $Q$ with $Q^\top Q = I$ such that $H^* := Q^\top H(\gamma^*) Q$ is 
negative definite for any $\gamma^*\in \Gamma^*$.
\end{assumption}

Assumption~\ref{ass: Hessian} is inspired by Assumption~7 of \cite{Zhouetal2026}.
It implies that there exists a unique $g^*$ satisfying $g^* = Q^\top \gamma^*$ for any $\gamma^* \in \Gamma^*$ (see Lemma~\ref{lem: unique_iden}~(i) in Appendix~\ref{append: proof}).
Hence under Assumption~\ref{ass: Hessian}, all the parameters in $\Gamma^*$ boil down to one vector after projection onto a low-dimensional space.

\medskip
\begin{theorem}[Rate of convergence]\label{thm: consistency_2_rate}
Under Assumptions~\ref{ass: compact}--\ref{ass: Hessian},
as $n\to \infty$, we have
$d(\widehat{\gamma}, \Gamma^*) = O_P\big(  n^{-1/2} + \|\Psi(\widehat{\gamma})\| \big)$,
where $\Psi(\widehat{\gamma}) := n^{-1} \sum_{i=1}^n \psi(\widehat{\gamma}; X_{i,\leq L})$ with $\{X_{i,\leq L}\}_{i=1}^n$ from Algorithm~\ref{alg:enpretrain_pro}.
\end{theorem}


Theorem~\ref{thm: consistency_2_rate} explicitly accounts for the interplay between statistical estimation and early‑stopped optimization.
Two sources of errors are explicitly separated:
(i) the $n^{-1/2}$ term as a statistical error coming from our method;
(ii) the order of $\|\Psi(\widehat{\gamma})\|$ reflecting the fact that $\widehat{\gamma}$ only approximately solves the estimating equation (e.g., when optimization is stopped early).
Hence, our expansion is informative when $\widehat{\gamma}$ is computed via SGD with a finite number of steps:
the term $\|\Psi(\widehat{\gamma})\|$ captures the optimization error and reveals how it propagates into the parameter estimate.
The asymptotic normality of our estimator can also be constructed; see Theorem~\ref{thm: normality_2} in Appendix~\ref{subsec: additional_thm}.

\section{Experiments}\label{sec:experiments}



We evaluate the performance of our proposed perturbation framework through experiments spanning synthetic scenarios and real-world applications. 

\subsection{Experiments on synthetic data}

\textbf{Datasets}. 
We construct synthetic datasets from a perturbed bigram language model. 
The ground-truth data-generating process is adapted from a standard bigram model, which is characterized by an initial token distribution and a token-to-token transition matrix $M_0\in\mathbb{R}^{|\mathcal{V}|\times |\mathcal{V}|}$, where $|\mathcal{V}|$ denotes the vocabulary size. 
We generate the initial token distribution uniformly over the vocabulary, and independently sample each row of the transition matrix $M_0$ from a Dirichlet distribution with concentration parameter \(0.5\). 
To introduce continuous perturbations, we additionally sample a Gaussian noise variable at each time step and use a fixed neural network \(T_{\beta_0}\) to generate an embedding-level perturbation based on the previous token embedding. 
The perturbed embedding is then used to generate the next token. 
The scalar parameter \(\alpha\) controls the perturbation strength, where \(\alpha=0\) corresponds to the standard bigram model without perturbation.
We sample \(500\) sequences from this perturbed transition process to form the training dataset \(\mathcal D\), where each sequence has maximum length \(10\). 
Additional details on the data generation process are provided in Appendix~\ref{appsec:sim_details}.

\textbf{Training and evaluation}.
We estimate the model by combining:
(i) $P_{\theta}$, a neural bigram language model \citep{zhang2023mixce} where the conditional distribution of the next token is parameterized by a feedforward neural network acting on the embedding of the previous token;
and (ii) $T_{\beta}$, a recurrent context encoder together with a multilayer perceptron that transforms a low-dimensional latent perturbation variable and the contextual embedding representation into a perturbation matrix in the embedding space.

Hereafter, we refer to the estimation procedure without the last term subtracted in \eqref{eqn: def_psi} as \textit{without debiasing}; otherwise as \textit{with debiasing}.
We implement our framework under three configurations: without debiasing, and with debiasing activated from the $10$-th and $20$-th optimization steps, respectively. 
As a competing baseline, we also implement the discrete perturbation method by \cite{Cenetal2026} with perturbation intensity parameters $\{0, 0.2, 0.6, 1\}$. 
More details regarding $P_{\theta}$, $T_{\beta}$, and the training procedures are included in Appendix~\ref{appsec:sim_details}. 
For each method, we estimate a transition matrix over the observed token space and compare it with the oracle matrix \(M_\alpha\),
and the extrapolation performance is evaluated by its mean absolute error (MAE) over token pairs unseen from the training corpus.

\textbf{Results}.
Figure~\ref{fig:sim_MAE} shows that the proposed trainable perturbation methods consistently outperform the discrete perturbation baselines across all vocabulary sizes. In addition, the continuous perturbation methods exhibit more stable MAEs as the vocabulary size changes, whereas the discrete perturbation methods are more sensitive to the perturbation intensity. 
For the proposed methods, the results suggest that several initial optimization steps without debiasing are helpful for providing a good warm start before activating debiasing. 
When the vocabulary size is moderate, debiasing generally improves performance,
whereas the vocabulary size becomes large, the warm start after $20$ steps may be insufficient.
Finally, although the proposed methods consistently maintain an advantage over the discrete perturbation baselines, the performance gap gradually decreases as the vocabulary size increases. One possible explanation is that, with a fixed embedding dimension, larger vocabularies make the embedding space less informative for smooth extrapolation.

\begin{figure}[ht]
    \centering
    \includegraphics[width = 1\linewidth]{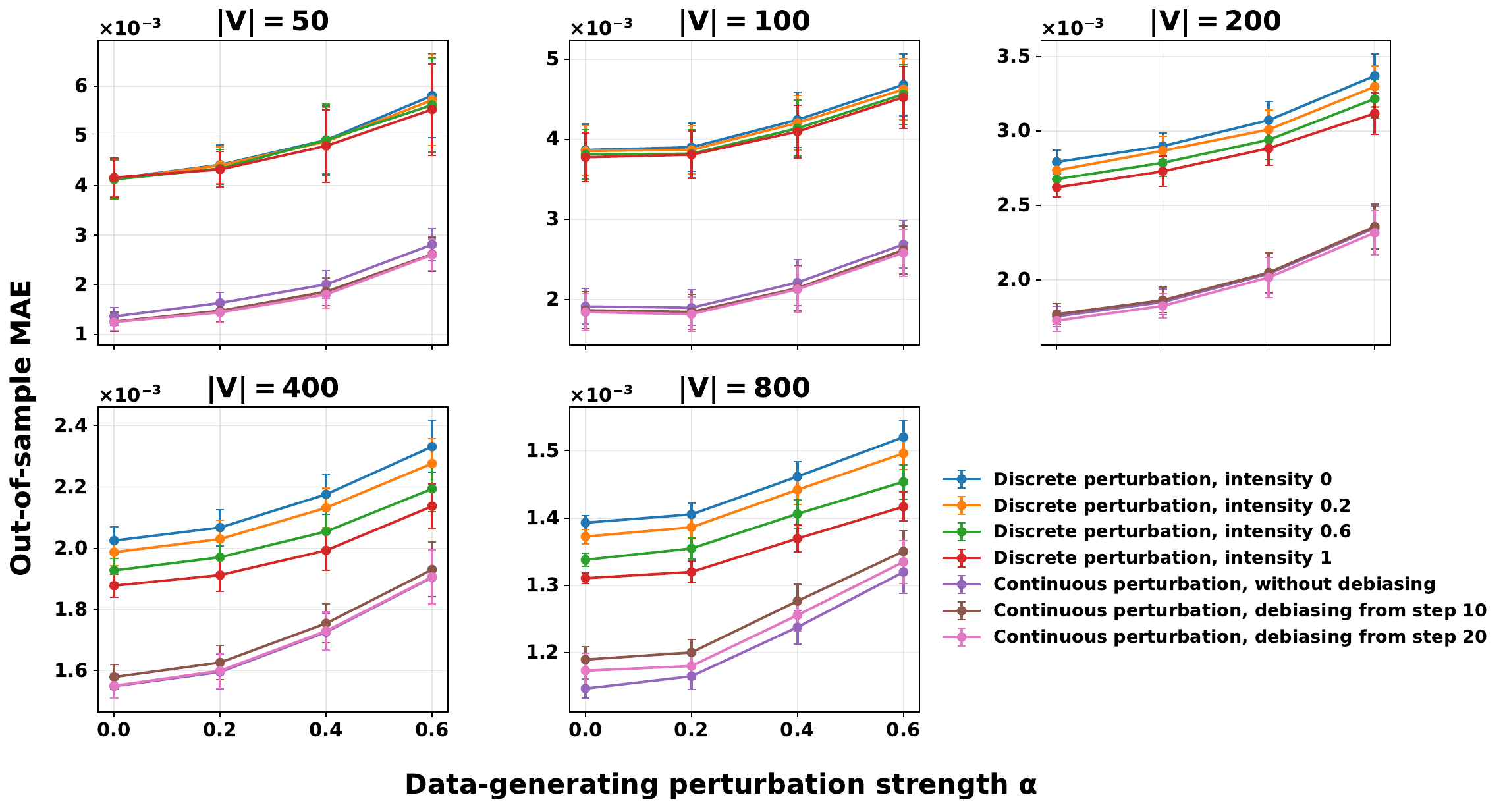}
    \caption{Out-of-sample MAEs of estimated transition matrices under varying data-generating perturbation strengths \(\alpha\). Panels correspond to different vocabulary sizes \(|\mathcal{V}|=50,100,200,400,800\). Results are reported as mean MAEs \(\pm 2\) standard errors over \(10\) independent replications.}
    \label{fig:sim_MAE}
\end{figure}

\subsection{Experiments on real-world data}\label{subsec:real_data}

\textbf{Datasets}. 
Following the experimental protocol in \cite{ren2024emo}, we train all models on the WikiText-2 corpus \citep{merity2017pointer}. To evaluate the generalization ability of the learned models, we assess the trained models on five external datasets spanning diverse domains and writing styles. Specifically, these datasets include: 
(i) WikiText-103 \citep{merity2017pointer}, a substantially larger Wikipedia-based benchmark containing over fifty times more tokens than WikiText-2; 
(ii) the test split of the WebText dataset introduced by OpenAI \citep{radford2019language}, consisting of web documents collected from outbound Reddit links; 
(iii) WritingPrompts \citep{fan2018hierarchical}, a long-form story generation dataset collected from Reddit writing prompt communities; 
(iv) CodeParrot \citep{xu2022systematic}, a large-scale source code dataset constructed from publicly available GitHub repositories; 
and (v) GermanQuAD \citep{moller2021germanquad}, a German question answering benchmark derived from Wikipedia articles. 
In addition, we evaluate the model performance on the WikiText-2 test split.
Markedly, the WikiText-2 test split and WikiText-103 are close to the training data, whereas the other datasets generally contain texts with different semantic contexts.

\textbf{Training and evaluation}. 
Three representative decoder-only Transformer architectures are experimented as base models: OPT \citep{zhang2022opt}, Qwen3 \citep{yang2025qwen3}, and GPT-Neo \citep{black2021gpt}. 
We implement the proposed continuous perturbation framework using a perturbation neural network consisting of an LSTM context encoder and a multilayer perceptron perturbation generator. 
We consider two configurations of the proposed method: without debiasing and with debiasing activated from the \(2000\)-th optimization step. 
We compare the proposed continuous perturbation framework with standard MLE, the discrete perturbation method of \cite{Cenetal2026}, and the embedding-level perturbation approach NEFTune \citep{jain2024neftune}. 
For further details on the implementation, we refer the reader to Appendix~\ref{appsec:real_details}. 
We primarily evaluate language modeling performance using perplexity \citep{jelinek1977perplexity}, where smaller values indicate more accurate next-token prediction and lower predictive uncertainty on the target distribution. 
To further assess the distributional similarity between generated text and human-written text, we report the Mauve score \citep{pillutla2023mauve}, which compares the two distributions through the area under the KL-divergence frontier. 
We also compute the ROUGE-1 score \citep{lin2004rouge} to quantify unigram overlap between generated outputs and reference texts; the results are relegated to Table~\ref{tab:r1_results} in Appendix~\ref{appsec:numerical}.

\textbf{Results}. 
Table~\ref{tab:PPL_mauve_results} reports the perplexity (PPL) and Mauve results across different base language models and evaluation datasets. 
Comparing the two configurations of our method, the results with debiasing generally achieve slightly lower perplexity when the base model is OPT. It also remains competitive and stable over all settings, suggesting that the debiasing correction improves the estimation accuracy and optimization stability of the perturbation objective.

More importantly, compared to other methods, our proposed framework consistently achieves substantially lower perplexity than the baseline methods in all settings.
Meanwhile, our methods often attain competitive or superior Mauve scores on out-of-domain datasets such as WritingPrompts and GermanQuAD.
These indicate improved language modeling and generalization capability. 
On the in-domain datasets WikiText-2 and WikiText-103, however, the continuous perturbation methods do not always achieve the highest Mauve scores, possibly because the perturbation encourages exploration beyond the original training distribution rather than optimizing similarity to the in-domain reference distribution.

\begin{table}[ht]
\centering
\caption{
Comparison of PPL and Mauve scores across different base language models and evaluation datasets. Lower PPL and higher Mauve indicate better performance. Best results within each base model are highlighted in bold. ``Vanilla'' denotes the original base model without perturbation, while ``w/o DB'' and ``w/ DB'' stand for without debiasing and with debiasing, respectively.
}
\small
\setlength{\tabcolsep}{4.5pt}
\renewcommand{\arraystretch}{1.15}

\resizebox{\textwidth}{!}{
\begin{tabular}{ll|cc|cc|cc|cc|cc|cc}
\toprule

\multirow{2}{*}{Model} & \multirow{2}{*}{Method}
& \multicolumn{2}{c|}{WikiText-2}
& \multicolumn{2}{c|}{WikiText-103}
& \multicolumn{2}{c|}{WebText}
& \multicolumn{2}{c|}{WritingPrompts}
& \multicolumn{2}{c|}{CodeParrot}
& \multicolumn{2}{c}{GermanQuAD} \\

\cline{3-14}

&& PPL & Mauve
& PPL & Mauve
& PPL & Mauve
& PPL & Mauve
& PPL & Mauve
& PPL & Mauve \\

\midrule

\multirow{5}{*}{OPT}

& Vanilla
& 145 & 0.77
& 145 & 0.77
& 226 & 0.33
& 164 & 0.13
& 602 & 0.17
& 604 & 0.47 \\

& Discrete
& 159 & \textbf{0.79}
& 159 & \textbf{0.79}
& 251 & 0.34
& 179 & 0.15
& 751 & 0.14
& 602 & 0.71 \\

& NEFTune
& 150 & \textbf{0.79}
& 150 & \textbf{0.79}
& 204 & 0.37
& 160 & 0.16
& 464 & 0.23
& 491 & 0.82 \\

& Cont. (w/o DB)
& 132 & 0.60
& 132 & 0.60
& 106 & 0.65
& 125 & \textbf{0.21}
& 293 & \textbf{0.26}
& 242 & \textbf{0.97} \\

& Cont. (w/ DB)
& \textbf{129} & 0.62
& \textbf{129} & 0.62
& \textbf{104} & \textbf{0.66}
& \textbf{121} & 0.20
& \textbf{289} & \textbf{0.26}
& \textbf{235} & 0.96 \\

\midrule

\multirow{5}{*}{Qwen3}

& Vanilla
& 97 & 0.76
& 97 & 0.76
& 156 & 0.43
& 120 & 0.17
& 40 & \textbf{0.50}
& 174 & \textbf{0.96} \\

& Discrete
& 113 & 0.75
& 113 & 0.75
& 191 & 0.38
& 148 & 0.21
& 43 & 0.36
& 206 & \textbf{0.96} \\

& NEFTune
& 95 & \textbf{0.77}
& 95 & \textbf{0.77}
& 129 & 0.58
& 113 & 0.28
& 27 & \textbf{0.50}
& 127 & 0.95 \\

& Cont. (w/o DB)
& \textbf{72} & 0.51
& \textbf{72} & 0.51
& \textbf{99} & \textbf{0.68}
& \textbf{93} & \textbf{0.32}
& \textbf{22} & 0.45
& \textbf{117} & 0.95 \\

& Cont. (w/ DB)
& 73 & 0.50
& 73 & 0.50
& \textbf{99} & 0.67
& \textbf{93} & 0.30
& \textbf{22} & 0.44
& \textbf{117} & 0.95 \\

\midrule

\multirow{5}{*}{GPT-Neo}

& Vanilla
& 78 & \textbf{0.83}
& 78 & \textbf{0.83}
& 110 & 0.47
& 101 & 0.24
& 55 & 0.18
& 189 & 0.94 \\

& Discrete
& 87 & 0.81
& 87 & 0.81
& 131 & 0.45
& 123 & 0.23
& 68 & 0.16
& 196 & 0.96 \\

& NEFTune
& 77 & 0.69
& 77 & 0.69
& 68 & 0.67
& 102 & \textbf{0.30}
& \textbf{26} & \textbf{0.74}
& 131 & 0.95 \\

& Cont. (w/o DB)
& 77 & 0.68
& 77 & 0.68
& \textbf{66} & \textbf{0.68}
& \textbf{97} & 0.29
& 28 & 0.57
& \textbf{127} & \textbf{0.97} \\

& Cont. (w/ DB)
& \textbf{76} & 0.68
& \textbf{76} & 0.68
& \textbf{66} & \textbf{0.68}
& \textbf{97} & 0.29
& 28 & 0.58
& \textbf{127} & 0.96 \\

\bottomrule
\end{tabular}
}

\label{tab:PPL_mauve_results}
\end{table}

\subsection{An ablation study}\label{subsec:ablation}

As our framework incorporates perturbations during both model training and inference, 
we examine the effect of perturbation at each stage by conducting an ablation study on the real-world datasets described in Section~\ref{subsec:real_data}. 
Specifically, we compare with two modified variants of our framework:
(i) TrainPerturb, where perturbations are introduced only during training; 
and (ii) TestPerturb, where the model is trained on the original data and perturbations are applied only during inference.

Figure~\ref{fig:ablation_gptneo} presents the differences in PPL and Mauve relative to the proposed continuous perturbation framework (Ours) on the GermanQuAD dataset using GPT-Neo as the base model for three methods: the standard model without perturbation (NoPerturb), and the two ablation variants described above (TrainPerturb and TestPerturb).
The vertical axis represents the difference of metric values between each method and Ours. 
Since lower PPL and higher Mauve indicate better generation quality, positive differences in PPL and negative differences in Mauve suggest improvements of the proposed method over the corresponding baselines. 
The results indicate that our framework achieves substantially lower PPL together with consistently higher Mauve scores than NoPerturb and TestPerturb, suggesting that incorporating perturbation during training is essential for improving extrapolative performance. 
Compared with TrainPerturb, our proposed method achieves better average PPL and Mauve performance, although the improvement is milder than those seen in NoPerturb and TestPerturb.
These observations suggest that applying perturbation only during training mainly acts as a form of data augmentation and cannot fully leverage the benefit of the perturbation mechanism. 
Similar patterns are also seen by using OPT as the base model,
as shown in Figure~\ref{fig:ablation_opt}.
Thus, jointly introducing perturbations during both training and inference yields better overall extrapolative generation performance.

\begin{figure}[!ht]
    \centering
    \begin{subfigure}{0.9\linewidth}
        \centering
        \includegraphics[width=\linewidth]
        {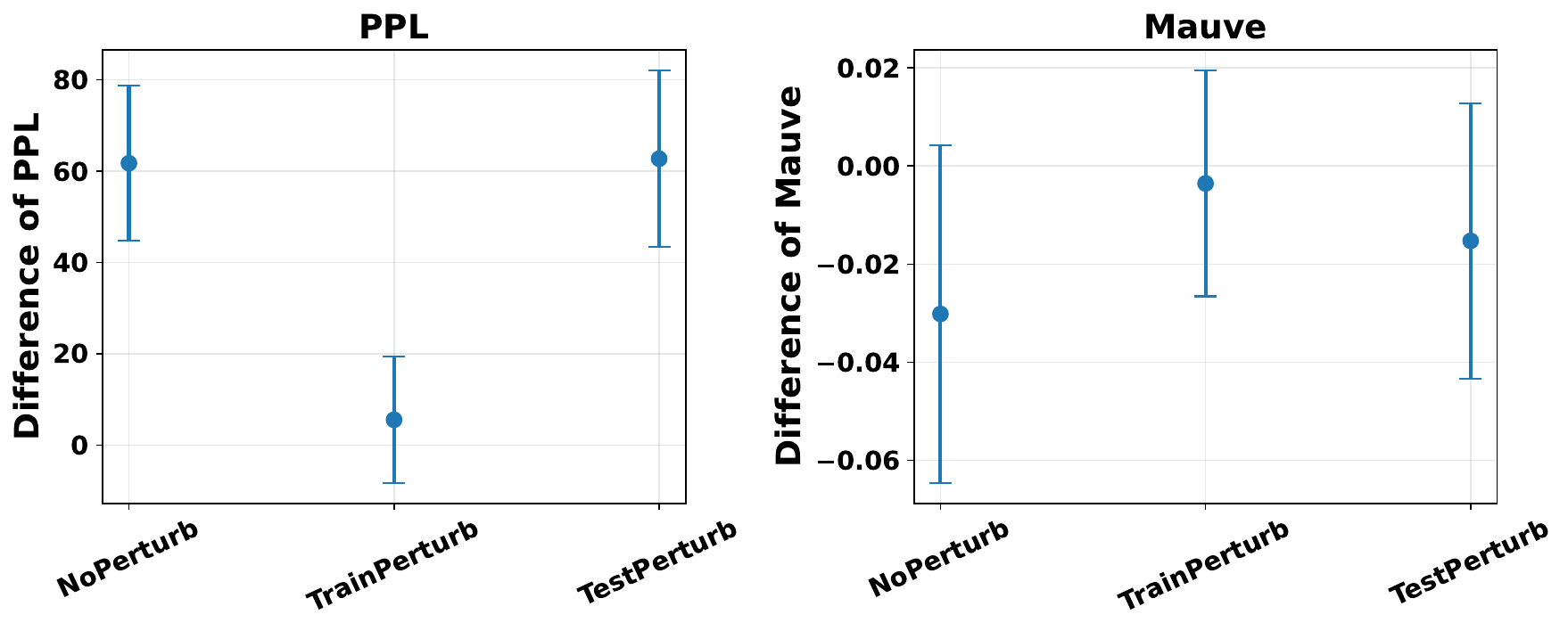}
        \caption{GPT-Neo as the base model.}
        \label{fig:ablation_gptneo}
    \end{subfigure}

    \vspace{0.8em}

    \begin{subfigure}{0.9\linewidth}
        \centering
        \includegraphics[width=\linewidth]
        {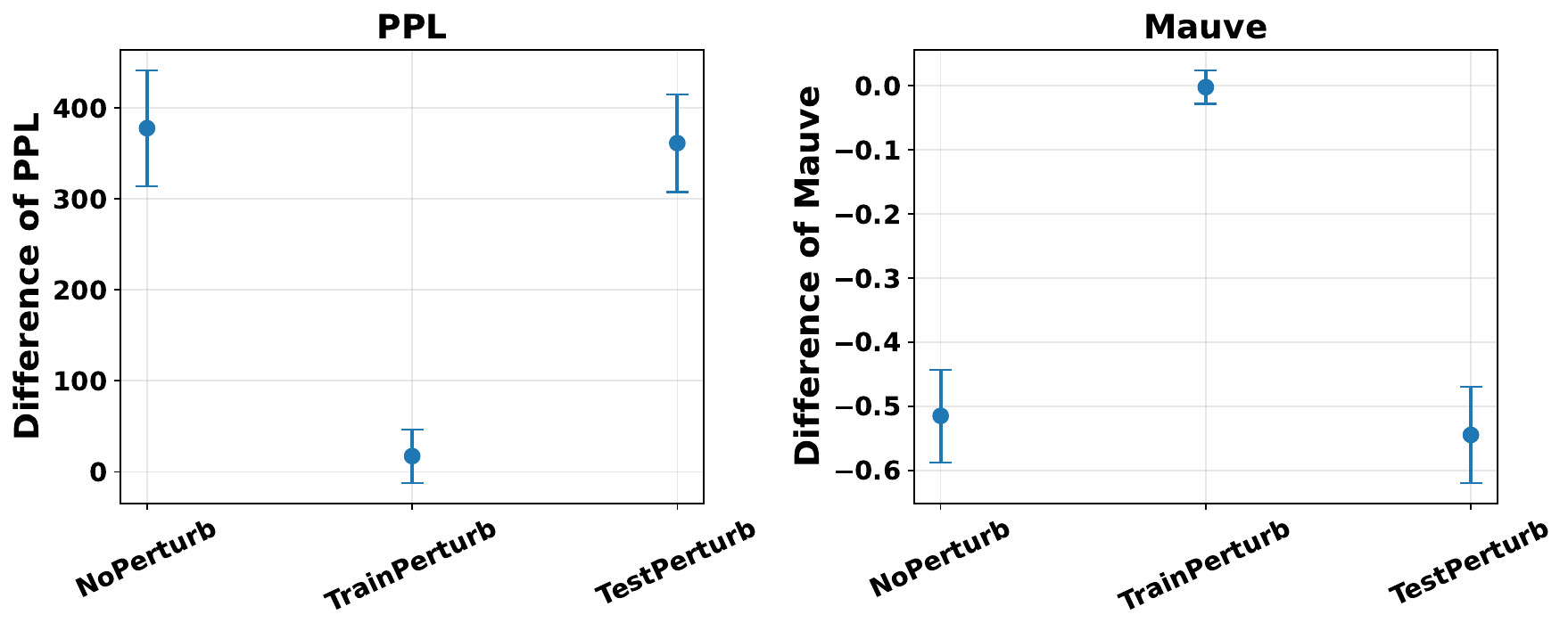}
        \caption{OPT as the base model.}
        \label{fig:ablation_opt}
    \end{subfigure}

    \caption{
    Real-data ablation study showing the differences in PPL and Mauve relative to the proposed method across $50$ independent repetitions. The models are trained on WikiText-2 and evaluated on the GermanQuAD dataset. Points denote the mean over repetitions, and error bars correspond to 95\% confidence intervals.
    }
    
    \label{fig:ablation}
\end{figure}

\section{Concluding Remarks}\label{sec:conclusion}

We introduce a trainable and continuous perturbation process for the perturbed autoregressive language models, by leveraging an unbiased estimating equation that circumvents the intractable marginal likelihood.
Under some regularity conditions, the estimator is consistent with a desired rate of convergence.
Numerical studies demonstrate strong extrapolation and out-of-domain generation performance across a range of language generation tasks and base language models.



\bibliographystyle{apalike}
\bibliography{ref}

@article{xu2026doubly,
  title={Doubly robust alignment for large language models},
  author={Xu, Erhan and Ye, Kai and Zhou, Hongyi and Zhu, Luhan and Quinzan, Francesco and Shi, Chengchun},
  journal={Advances in Neural Information Processing Systems},
  volume={38},
  pages={6741--6790},
  year={2025}
}

@article{xie2026statistical,
  title={Statistical Early Stopping for Reasoning Models},
  author={Xie, Yangxinyu and Wang, Tao and Mallick, Soham and Sun, Yan and Noarov, Georgy and Yu, Mengxin and Mallick, Tanwi and Su, Weijie J and Dobriban, Edgar},
  journal={arXiv preprint arXiv:2602.13935},
  year={2026}
}

@article{shao2024deepseekmath,
  title={Deepseekmath: Pushing the limits of mathematical reasoning in open language models},
  author={Shao, Zhihong and Wang, Peiyi and Zhu, Qihao and Xu, Runxin and Song, Junxiao and Bi, Xiao and Zhang, Haowei and Zhang, Mingchuan and Li, YK and Wu, Yang and others},
  journal={arXiv preprint arXiv:2402.03300},
  year={2024}
}

@article{xia2026statistical,
  title={A statistical framework for alignment with biased ai feedback},
  author={Xia, Xintao and Xia, Zhiqiu and Zhang, Linjun and Cai, Zhanrui},
  journal={arXiv preprint arXiv:2602.08259},
  year={2026}
}

@article{rafailov2023direct,
  title={Direct preference optimization: Your language model is secretly a reward model},
  author={Rafailov, Rafael and Sharma, Archit and Mitchell, Eric and Manning, Christopher D and Ermon, Stefano and Finn, Chelsea},
  journal={Advances in neural information processing systems},
  volume={36},
  pages={53728--53741},
  year={2023}
}

@article{liu2024dual,
  title={Dual active learning for reinforcement learning from human feedback},
  author={Liu, Pangpang and Shi, Chengchun and Sun, Will Wei},
  journal={arXiv preprint arXiv:2410.02504},
  year={2024}
}

@article{NeweyMcFadden1994,
  title={Large sample estimation and hypothesis testing},
  author={Newey, Whitney K and McFadden, Daniel},
  journal={Handbook of econometrics},
  volume={4},
  pages={2111--2245},
  year={1994},
  publisher={Elsevier}
}

@book{Vaart1998, 
place={Cambridge}, 
series={Cambridge Series in Statistical and Probabilistic Mathematics}, 
title={Asymptotic Statistics}, 
publisher={Cambridge University Press}, 
author={Vaart, A. W. van der}, 
year={1998}, 
collection={Cambridge Series in Statistical and Probabilistic Mathematics}
}

@article{Zhouetal2026,
  title={Demystifying Group Relative Policy Optimization: Its Policy Gradient is a U-Statistic},
  author={Hongyi Zhou and Kai Ye and Erhan Xu and Jin Zhu and Ying Yang and Shijin Gong and Chengchun Shi},
  journal={arXiv preprint arXiv:2603.01162v3},
  year={2026}
}

@book{VaartWellner2023,
  title={Weak Convergence and Empirical Processes},
  author={A. W. van der Vaart and Jon A. Wellner},
  year={2023},
  publisher={Springer Cham}
}

@inproceedings{zhang2023mixce,
  title={MixCE: Training autoregressive language models by mixing forward and reverse cross-entropies},
  author={Zhang, Shiyue and Wu, Shijie and Irsoy, Ozan and Lu, Steven and Bansal, Mohit and Dredze, Mark and Rosenberg, David},
  booktitle={Proceedings of the 61st Annual Meeting of the Association for Computational Linguistics (Volume 1: Long Papers)},
  pages={9027--9050},
  year={2023}
}

@techreport{radford2018improving,
  title={Improving language understanding by generative pre-training},
  author={Radford, Alec and Narasimhan, Karthik and Salimans, Tim and Sutskever, Ilya},
  year={2018},
  institution={OpenAI}
}

@techreport{radford2019language,
  title={Language models are unsupervised multitask learners},
  author={Radford, Alec and Wu, Jeff and Child, Rewon and Luan, David and Amodei, Dario and Sutskever, Ilya},
  year={2019},
  institution={OpenAI}
}

@article{zhang2022opt,
      title={OPT: Open Pre-trained Transformer Language Models}, 
      author={Susan Zhang and Stephen Roller and Naman Goyal and Mikel Artetxe and Moya Chen and Shuohui Chen and Christopher Dewan and Mona Diab and Xian Li and Xi Victoria Lin and Todor Mihaylov and Myle Ott and Sam Shleifer and Kurt Shuster and Daniel Simig and Punit Singh Koura and Anjali Sridhar and Tianlu Wang and Luke Zettlemoyer},
      year={2022},
      journal={arXiv preprint arXiv:2205.01068v4}
}

@inproceedings{ji2023tailoring,
    title={Tailoring Language Generation Models under Total Variation Distance},
    author={Haozhe Ji and Pei Ke and Zhipeng Hu and Rongsheng Zhang and Minlie Huang},
    booktitle={The Eleventh International Conference on Learning Representations },
    year={2023}
}

@inproceedings{ren2024emo,
title={{EMO}: EARTH MOVER DISTANCE OPTIMIZATION FOR AUTO-REGRESSIVE LANGUAGE MODELING},
author={Siyu Ren and Zhiyong Wu and Kenny Q. Zhu},
booktitle={The Twelfth International Conference on Learning Representations},
year={2024}
}

@inproceedings{feng2020genaug,
    title = "{G}en{A}ug: Data Augmentation for Finetuning Text Generators",
    author = "Feng, Steven Y.  and
      Gangal, Varun  and
      Kang, Dongyeop  and
      Mitamura, Teruko  and
      Hovy, Eduard",
    editor = "Agirre, Eneko  and
      Apidianaki, Marianna  and
      Vuli{\'c}, Ivan",
    booktitle = "Proceedings of Deep Learning Inside Out (DeeLIO): The First Workshop on Knowledge Extraction and Integration for Deep Learning Architectures",
    month = nov,
    year = "2020",
    address = "Online",
    publisher = "Association for Computational Linguistics",
    doi = "10.18653/v1/2020.deelio-1.4",
    pages = "29--42",
}

@inproceedings{feng2019keep,
    title = "Keep Calm and Switch On! Preserving Sentiment and Fluency in Semantic Text Exchange",
    author = "Feng, Steven Y.  and
      Li, Aaron W.  and
      Hoey, Jesse",
    editor = "Inui, Kentaro  and
      Jiang, Jing  and
      Ng, Vincent  and
      Wan, Xiaojun",
    booktitle = "Proceedings of the 2019 Conference on Empirical Methods in Natural Language Processing and the 9th International Joint Conference on Natural Language Processing (EMNLP-IJCNLP)",
    month = nov,
    year = "2019",
    address = "Hong Kong, China",
    publisher = "Association for Computational Linguistics",
    doi = "10.18653/v1/D19-1272",
    pages = "2701--2711",
}

@inproceedings{wei2019eda,
    title = "{EDA}: Easy Data Augmentation Techniques for Boosting Performance on Text Classification Tasks",
    author = "Wei, Jason  and
      Zou, Kai",
    editor = "Inui, Kentaro  and
      Jiang, Jing  and
      Ng, Vincent  and
      Wan, Xiaojun",
    booktitle = "Proceedings of the 2019 Conference on Empirical Methods in Natural Language Processing and the 9th International Joint Conference on Natural Language Processing (EMNLP-IJCNLP)",
    month = nov,
    year = "2019",
    address = "Hong Kong, China",
    publisher = "Association for Computational Linguistics",
    doi = "10.18653/v1/D19-1670",
    pages = "6382--6388"
}

@inproceedings{merity2017pointer,
title={Pointer sentinel mixture models},
author={Merity, Stephen and Xiong, Caiming and Bradbury, James and Socher, Richard},
booktitle={5th International Conference on Learning Representations},
year={2017}
}

@inproceedings{fan2018hierarchical,
  title={Hierarchical neural story generation},
  author={Fan, Angela and Lewis, Mike and Dauphin, Yann},
  booktitle={Proceedings of the 56th Annual Meeting of the Association for Computational Linguistics (Volume 1: Long Papers)},
  pages={889--898},
  year={2018}
}

@inproceedings{xu2022systematic,
  title={A systematic evaluation of large language models of code},
  author={Xu, Frank F and Alon, Uri and Neubig, Graham and Hellendoorn, Vincent Josua},
  booktitle={Proceedings of the 6th ACM SIGPLAN international symposium on machine programming},
  pages={1--10},
  year={2022}
}

@inproceedings{moller2021germanquad,
  title={GermanQuAD and GermanDPR: Improving non-English question answering and passage retrieval},
  author={M{\"o}ller, Timo and Risch, Julian and Pietsch, Malte},
  booktitle={Proceedings of the 3rd workshop on machine reading for question answering},
  pages={42--50},
  year={2021}
}

@inproceedings{kingma2015adam,
title={Adam: A method for stochastic optimization},
author={Kingma, Diederik P and Ba, Jimmy},
booktitle={3rd International Conference on Learning Representations},
year={2015}
}

@article{black2021gpt,
  title={Gpt-neo: Large scale autoregressive language modeling with mesh-tensorflow},
  author={Black, Sid and Leo, Gao and Wang, Phil and Leahy, Connor and Biderman, Stella},
  journal={Zenodo},
  year={2021}
}

@article{Cenetal2026,
  title={Perturbation is All You Need for Extrapolating Language Models},
  author={Zetai Cen and Jin Zhu and Xinwei Shen and Chengchun Shi},
  journal={arXiv preprint arXiv:2605.04344v1},
  year={2026}
}

@inproceedings{Lietal2024,
title={Empowering Large Language Models for Textual Data Augmentation},
author={Yichuan Li and Kaize Ding and Jianling Wang and Kyumin Lee},
booktitle={ICLR 2024 Workshop on Data-centric Machine Learning Research (DMLR): Harnessing Momentum for Science},
year={2024},
url={https://openreview.net/forum?id=JCuWXk9HGq}
}

@article{LiangZeger1986,
    author = {Liang, Kung-Yee and Zeger, Scott L.},
    title = {Longitudinal data analysis using generalized linear models},
    journal = {Biometrika},
    volume = {73},
    number = {1},
    pages = {13-22},
    year = {1986},
    month = {04},
    issn = {0006-3444},
    doi = {10.1093/biomet/73.1.13},
    url = {https://doi.org/10.1093/biomet/73.1.13},
    eprint = {https://academic.oup.com/biomet/article-pdf/73/1/13/679793/73-1-13.pdf},
}

@article{Songetal2024,
author = {Shanshan Song and Yuanyuan Lin and Yong Zhou},
title = {A General M-estimation Theory in Semi-Supervised Framework},
journal = {Journal of the American Statistical Association},
volume = {119},
number = {546},
pages = {1065--1075},
year = {2024},
publisher = {Taylor \& Francis},
doi = {10.1080/01621459.2023.2169699},
URL = { https://doi.org/10.1080/01621459.2023.2169699
},
eprint = { https://doi.org/10.1080/01621459.2023.2169699
}
}

@article{YanChen2024,
  title={Transfer Learning with General Estimating Equations},
  author={Han Yan and Song Xi Chen},
  journal={arXiv preprint arXiv:2410.04398v1},
  year={2024}
}

@article{Hyvarinen2005,
  author  = {Aapo Hyv{{\"a}}rinen},
  title   = {Estimation of Non-Normalized Statistical Models by Score Matching},
  journal = {Journal of Machine Learning Research},
  year    = {2005},
  volume  = {6},
  number  = {24},
  pages   = {695--709},
  url     = {http://jmlr.org/papers/v6/hyvarinen05a.html}
}

@InProceedings{GutmannHyvarinen2010,
  title = 	 {Noise-contrastive estimation: A new estimation principle for unnormalized statistical models},
  author = 	 {Gutmann, Michael and Hyvärinen, Aapo},
  booktitle = 	 {Proceedings of the Thirteenth International Conference on Artificial Intelligence and Statistics},
  pages = 	 {297--304},
  year = 	 {2010},
  editor = 	 {Teh, Yee Whye and Titterington, Mike},
  volume = 	 {9},
  series = 	 {Proceedings of Machine Learning Research},
  address = 	 {Chia Laguna Resort, Sardinia, Italy},
  month = 	 {13--15 May},
  publisher =    {PMLR},
  url = 	 {https://proceedings.mlr.press/v9/gutmann10a.html}
}

@inproceedings{Chenetal2024,
title={Noise Contrastive Alignment of Language Models with Explicit Rewards},
author={Huayu Chen and Guande He and Lifan Yuan and Ganqu Cui and Hang Su and Jun Zhu},
booktitle={The Thirty-eighth Annual Conference on Neural Information Processing Systems},
year={2024},
url={https://openreview.net/forum?id=KwRLDkyVOl}
}

@inproceedings{Songetal2021,
  author       = {Yang Song and
                  Jascha Sohl{-}Dickstein and
                  Diederik P. Kingma and
                  Abhishek Kumar and
                  Stefano Ermon and
                  Ben Poole},
  title        = {Score-Based Generative Modeling through Stochastic Differential Equations},
  booktitle    = {9th International Conference on Learning Representations, {ICLR} 2021,
                  Virtual Event, Austria, May 3-7, 2021},
  publisher    = {OpenReview.net},
  year         = {2021},
  url          = {https://openreview.net/forum?id=PxTIG12RRHS},
  bibsource    = {dblp computer science bibliography, https://dblp.org}
}

@inproceedings{Louetal2024,
author = {Lou, Aaron and Meng, Chenlin and Ermon, Stefano},
title = {Discrete diffusion modeling by estimating the ratios of the data distribution},
year = {2024},
publisher = {JMLR.org},
booktitle = {Proceedings of the 41st International Conference on Machine Learning},
articleno = {1333},
numpages = {30},
location = {Vienna, Austria},
series = {ICML'24}
}

@inproceedings{jain2024neftune,
  title={Neftune: Noisy embeddings improve instruction finetuning},
  author={Jain, Neel and Chiang, Ping-yeh and Wen, Yuxin and Kirchenbauer, John and Chu, Hong-Min and Somepalli, Gowthami and Bartoldson, Brian and Kailkhura, Bhavya and Schwarzschild, Avi and Saha, Aniruddha and others},
  booktitle={International Conference on Learning Representations},
  volume={2024},
  pages={17566--17591},
  year={2024}
}

@techreport{yang2025qwen3,
  title={Qwen3 Technical Report},
  author={Yang, An and Li, Anfeng and Yang, Baosong and Zhang, Beichen and Hui, Binyuan and Zheng, Bo and Yu, Bowen and Gao, Chang and Huang, Chengen and Lv, Chenxu and others},
  year={2025},
  institution={Alibaba Group},
  note={arXiv:2505.09388}
}

@article{pillutla2023mauve,
  title={Mauve scores for generative models: Theory and practice},
  author={Pillutla, Krishna and Liu, Lang and Thickstun, John and Welleck, Sean and Swayamdipta, Swabha and Zellers, Rowan and Oh, Sewoong and Choi, Yejin and Harchaoui, Zaid},
  journal={Journal of Machine Learning Research},
  volume={24},
  number={356},
  pages={1--92},
  year={2023}
}

@inproceedings{lin2004rouge,
  title={Rouge: A package for automatic evaluation of summaries},
  author={Lin, Chin-Yew},
  booktitle={Text summarization branches out},
  pages={74--81},
  year={2004}
}

@article{jelinek1977perplexity,
  title={Perplexity—a measure of the difficulty of speech recognition tasks},
  author={Jelinek, Fred and Mercer, Robert L and Bahl, Lalit R and Baker, James K},
  journal={The journal of the Acoustical Society of America},
  volume={62},
  number={S1},
  pages={S63--S63},
  year={1977},
  publisher={Acoustical Society of America}
}

@inproceedings{loshchilov2019decoupled,
  title={Decoupled weight decay regularization},
  author={Loshchilov, Ilya and Hutter, Frank},
  booktitle={International Conference on Learning Representations},
  year={2019}
}

@article{openai2025gpt,
  title={Openai gpt-5 system card},
  author={Singh, Aaditya and Fry, Adam and Perelman, Adam and Tart, Adam and Ganesh, Adi and El-Kishky, Ahmed and McLaughlin, Aidan and Low, Aiden and Ostrow, AJ and Ananthram, Akhila and others},
  journal={arXiv preprint arXiv:2601.03267},
  year={2025}
}

@article{xiao2025algorithmic,
  title     = {On the Algorithmic Bias of Aligning Large Language Models with {RLHF}: Preference Collapse and Matching Regularization},
  author    = {Xiao, Jiancong and Li, Ziniu and Xie, Xingyu and Getzen, Emily and Fang, Cong and Long, Qi and Su, Weijie},
  journal   = {Journal of the American Statistical Association},
  volume    = {120},
  number    = {552},
  pages     = {2154--2164},
  year      = {2025},
  publisher = {Taylor \& Francis}
}

@article{gong2026kernelized,
  title={Kernelized Advantage Estimation: From Nonparametric Statistics to LLM Reasoning},
  author={Gong, Shijin and Ye, Kai and Zhu, Jin and Zhang, Xinyu and Zhou, Hongyi and Shi, Chengchun},
  journal={arXiv preprint arXiv:2604.28005},
  year={2026}
}

@article{huang2026learning,
  title={On the Learning Dynamics of RLVR at the Edge of Competence},
  author={Huang, Yu and Wen, Zixin and Chi, Yuejie and Wei, Yuting and Singh, Aarti and Liang, Yingbin and Chen, Yuxin},
  journal={arXiv preprint arXiv:2602.14872},
  year={2026}
}

@inproceedings{eikema2020map,
  title={Is MAP decoding all you need? the inadequacy of the mode in neural machine translation},
  author={Eikema, Bryan and Aziz, Wilker},
  booktitle={Proceedings of the 28th International Conference on Computational Linguistics},
  pages={4506--4520},
  year={2020}
}

@inproceedings{Lietal2023_perturbscore,
    title = "{P}erturb{S}core: Connecting Discrete and Continuous Perturbations in {NLP}",
    author = "Li, Linyang  and
      Ren, Ke  and
      Shao, Yunfan  and
      Wang, Pengyu  and
      Qiu, Xipeng",
    editor = "Bouamor, Houda  and
      Pino, Juan  and
      Bali, Kalika",
    booktitle = "Findings of the Association for Computational Linguistics: EMNLP 2023",
    month = dec,
    year = "2023",
    address = "Singapore",
    publisher = "Association for Computational Linguistics",
    url = "https://aclanthology.org/2023.findings-emnlp.442/",
    doi = "10.18653/v1/2023.findings-emnlp.442",
    pages = "6638--6648"
}

@inproceedings{Parketal2022,
    title = "Consistency Training with Virtual Adversarial Discrete Perturbation",
    author = "Park, Jungsoo  and
      Kim, Gyuwan  and
      Kang, Jaewoo",
    editor = "Carpuat, Marine  and
      de Marneffe, Marie-Catherine  and
      Meza Ruiz, Ivan Vladimir",
    booktitle = "Proceedings of the 2022 Conference of the North American Chapter of the Association for Computational Linguistics: Human Language Technologies",
    month = jul,
    year = "2022",
    address = "Seattle, United States",
    publisher = "Association for Computational Linguistics",
    url = "https://aclanthology.org/2022.naacl-main.414/",
    doi = "10.18653/v1/2022.naacl-main.414",
    pages = "5646--5656"
}

@article{IsmailovAsanova2025,
  title={Small Edits, Big Consequences: Telling Good from Bad Robustness in Large Language Models},
  author={Altynbek Ismailov and Salia Asanova},
  journal={arXiv preprint arXiv:2507.15868v1},
  year={2025}
}

@inproceedings{Ouyangetal2022,
author = {Ouyang, Long and Wu, Jeff and Jiang, Xu and Almeida, Diogo and Wainwright, Carroll L. and Mishkin, Pamela and Zhang, Chong and Agarwal, Sandhini and Slama, Katarina and Ray, Alex and Schulman, John and Hilton, Jacob and Kelton, Fraser and Miller, Luke and Simens, Maddie and Askell, Amanda and Welinder, Peter and Christiano, Paul and Leike, Jan and Lowe, Ryan},
title = {Training language models to follow instructions with human feedback},
year = {2022},
isbn = {9781713871088},
publisher = {Curran Associates Inc.},
address = {Red Hook, NY, USA},
booktitle = {Proceedings of the 36th International Conference on Neural Information Processing Systems},
numpages = {15},
location = {New Orleans, LA, USA},
series = {NIPS '22}
}

@inproceedings{Iyyeretal2018,
    title = "Adversarial Example Generation with Syntactically Controlled Paraphrase Networks",
    author = "Iyyer, Mohit  and
      Wieting, John  and
      Gimpel, Kevin  and
      Zettlemoyer, Luke",
    editor = "Walker, Marilyn  and
      Ji, Heng  and
      Stent, Amanda",
    booktitle = "Proceedings of the 2018 Conference of the North {A}merican Chapter of the Association for Computational Linguistics: Human Language Technologies, Volume 1 (Long Papers)",
    month = jun,
    year = "2018",
    address = "New Orleans, Louisiana",
    publisher = "Association for Computational Linguistics",
    url = "https://aclanthology.org/N18-1170/",
    doi = "10.18653/v1/N18-1170",
    pages = "1875--1885"
}

@inproceedings{Ribeiroetal2018,
    title = "Semantically Equivalent Adversarial Rules for Debugging {NLP} models",
    author = "Ribeiro, Marco Tulio  and
      Singh, Sameer  and
      Guestrin, Carlos",
    editor = "Gurevych, Iryna  and
      Miyao, Yusuke",
    booktitle = "Proceedings of the 56th Annual Meeting of the Association for Computational Linguistics (Volume 1: Long Papers)",
    month = jul,
    year = "2018",
    address = "Melbourne, Australia",
    publisher = "Association for Computational Linguistics",
    url = "https://aclanthology.org/P18-1079/",
    doi = "10.18653/v1/P18-1079",
    pages = "856--865"
}

@inproceedings{Renetal2019generating,
    title = "Generating Natural Language Adversarial Examples through Probability Weighted Word Saliency",
    author = "Ren, Shuhuai  and
      Deng, Yihe  and
      He, Kun  and
      Che, Wanxiang",
    editor = "Korhonen, Anna  and
      Traum, David  and
      M{\`a}rquez, Llu{\'i}s",
    booktitle = "Proceedings of the 57th Annual Meeting of the Association for Computational Linguistics",
    month = jul,
    year = "2019",
    address = "Florence, Italy",
    publisher = "Association for Computational Linguistics",
    url = "https://aclanthology.org/P19-1103/",
    doi = "10.18653/v1/P19-1103",
    pages = "1085--1097"
}

@inproceedings{Zangetal2020word,
    title = "Word-level Textual Adversarial Attacking as Combinatorial Optimization",
    author = "Zang, Yuan  and
      Qi, Fanchao  and
      Yang, Chenghao  and
      Liu, Zhiyuan  and
      Zhang, Meng  and
      Liu, Qun  and
      Sun, Maosong",
    editor = "Jurafsky, Dan  and
      Chai, Joyce  and
      Schluter, Natalie  and
      Tetreault, Joel",
    booktitle = "Proceedings of the 58th Annual Meeting of the Association for Computational Linguistics",
    month = jul,
    year = "2020",
    address = "Online",
    publisher = "Association for Computational Linguistics",
    url = "https://aclanthology.org/2020.acl-main.540/",
    doi = "10.18653/v1/2020.acl-main.540",
    pages = "6066--6080"
}

@inproceedings{pruthi2019combating,
  title={Combating adversarial misspellings with robust word recognition},
  author={Pruthi, Danish and Dhingra, Bhuwan and Lipton, Zachary C},
  booktitle={Proceedings of the 57th Annual Meeting of the Association for Computational Linguistics},
  pages={5582--5591},
  year={2019}
}

@inproceedings{Zhu2020FreeLB,
title={FreeLB: Enhanced Adversarial Training for Natural Language Understanding},
author={Chen Zhu and Yu Cheng and Zhe Gan and Siqi Sun and Tom Goldstein and Jingjing Liu},
booktitle={International Conference on Learning Representations},
year={2020},
url={https://openreview.net/forum?id=BygzbyHFvB}
}

\clearpage
\appendix

\renewcommand{\thesection}{\Alph{section}}
\renewcommand{\theHsection}{\Alph{section}} 
\renewcommand{\thesubsection}{\Alph{section}.\arabic{subsection}}
\renewcommand{\theHsubsection}{\Alph{section}.\arabic{subsection}} 

\setcounter{table}{0}
\renewcommand{\thetable}{A\arabic{table}}
\setcounter{figure}{0}
\renewcommand{\thefigure}{A\arabic{figure}}

\begin{center}
{\LARGE\textbf{Supplement to ``Learning Perturbations to Extrapolate Your LLM''}}
\end{center}

\medskip

\section{Additional Numerical Results}\label{appsec:numerical}

\textbf{ROUGE-1 results in real data analysis}. 
Table~\ref{tab:r1_results} summarizes the ROUGE-1 results on the real-world language datasets described in Section~\ref{subsec:real_data}. The proposed continuous perturbation framework achieves competitive performance across different base language models and evaluation datasets.

Although the improvements in ROUGE-1 are generally more moderate than those in PPL and Mauve, the proposed methods consistently achieve competitive or superior ROUGE-1 scores across most evaluation settings, especially on out-of-domain datasets. These results suggest that the proposed framework improves not only likelihood modeling but also the lexical consistency and overall quality of generated continuations.

\begin{table}[!ht]
\centering
\caption{
ROUGE-1 scores across different base language models and evaluation datasets. Higher values indicate better performance. Best results within each base model and dataset are highlighted in bold.
Refer to the caption of Table~\ref{tab:PPL_mauve_results} for the explanation to notations.
}
\small
\setlength{\tabcolsep}{5pt}
\renewcommand{\arraystretch}{1.15}

\resizebox{\textwidth}{!}{
\begin{tabular}{ll|cccccc}
\toprule
Model & Method 
& WikiText-2 & WikiText-103 & WebText & WritingPrompts & CodeParrot & GermanQuAD \\
\midrule

\multirow{5}{*}{OPT}
& Vanilla        & \textbf{0.382} & \textbf{0.382} & 0.362          & \textbf{0.341} & 0.324          & 0.310 \\
& Discrete       & 0.381          & 0.381          & 0.363          & 0.340          & 0.321          & 0.320 \\
& NEFTune        & 0.381          & 0.381          & 0.365          & 0.339          & 0.334          & 0.323 \\
& Cont. (w/o DB) & 0.372          & 0.372          & \textbf{0.369} & 0.339          & \textbf{0.336} & \textbf{0.341} \\
& Cont. (w/ DB)  & 0.370          & 0.370          & \textbf{0.369} & 0.340          & \textbf{0.336} & \textbf{0.341} \\

\midrule

\multirow{5}{*}{Qwen3}
& Vanilla        & \textbf{0.391} & \textbf{0.391} & 0.371          & 0.344          & 0.448          & 0.355 \\
& Discrete       & 0.388          & 0.388          & 0.370          & 0.343          & 0.442          & 0.351 \\
& NEFTune        & 0.388          & 0.388          & \textbf{0.373} & 0.344          & 0.484          & \textbf{0.361} \\
& Cont. (w/o DB) & 0.384          & 0.384          & 0.372          & \textbf{0.346} & 0.494          & 0.360 \\
& Cont. (w/ DB)  & 0.384          & 0.384          & \textbf{0.373} & 0.345          & \textbf{0.495} & 0.360 \\

\midrule

\multirow{5}{*}{GPT-Neo}
& Vanilla        & \textbf{0.390}             & \textbf{0.390}             & 0.374             & \textbf{0.344}             & 0.409             & 0.342 \\
& Discrete       & \textbf{0.390} & \textbf{0.390} & 0.374          & 0.342          & 0.404          & 0.317 \\
& NEFTune        & 0.379          & 0.379          & \textbf{0.375} & 0.343          & 0.445          & \textbf{0.346} \\
& Cont. (w/o DB) & 0.378          & 0.378          & \textbf{0.375} & \textbf{0.344} & \textbf{0.452} & 0.344 \\
& Cont. (w/ DB)  & 0.378          & 0.378          & \textbf{0.375} & \textbf{0.344} & 0.451          & 0.323 \\

\bottomrule
\end{tabular}
}

\label{tab:r1_results}
\end{table}



\textbf{Examples of text generation}. 
We showcase several examples of generated texts from the real data analysis described in Section~\ref{subsec:real_data}. Table~\ref{tab:generated_examples_ascii} presents generations from different methods using GPT-Neo-1.3B as the base model, trained on WikiText-2 and evaluated on the GermanQuAD dataset. 
For this highly out-of-domain text generation task, the outputs generated by Vanilla and Discrete exhibit substantial semantic drift and severe degeneration behaviors. In particular, the generated texts contain many unnatural phrases, fragmented expressions, and unexpected symbols or punctuations such as ``@-@'', ``*00 00'', and irregular token combinations that are inconsistent with the original German context. The generated contents also deviate significantly from the original topic of ASCII character encoding and fail to maintain coherent semantic structures. 
Although NEFTune preserves part of the original topic by retaining terms related to ``Information Interchange'' and ``7-Bit-Zeichenkodierung'', it still produces mixed-language expressions involving Spanish phrases such as ``Codigo de los Estandardales de Información Interchange'', English expressions such as ``Canada's Universal'', and Czech-related terms such as ``Čeština'' within the same continuation, together with incoherent descriptions and unstable grammatical structures. 
In contrast, the proposed continuous perturbation methods generate texts that remain more semantically related to the original topic and exhibit relatively better structural coherence. 
Compared with the baseline methods, the generated texts contain fewer unexpected symbols and fragmented token patterns, while maintaining more stable sentence structures and topic consistency throughout the continuation.

\begin{table}[!ht]
\centering
\caption{Examples of text generated by different methods on the GermanQuAD dataset.}
\label{tab:generated_examples_ascii}
\small
\setlength{\tabcolsep}{3pt}
\renewcommand{\arraystretch}{1.15}

\begin{tabularx}{0.96\textwidth}{>{\bfseries}p{0.18\textwidth}X}
\toprule

\multicolumn{2}{c}{\cellcolor{blue!20}\textbf{Human Source}} \\
\midrule

Input text &
\foreignlanguage{ngerman}{American\_Standard\_Code\_for\_Information\_Interchange Der (ASCII, alternativ US-ASCII, oft ausgesprochen, ) ist eine 7-Bit-Zeichenkodierung; sie entspricht der US-Variante von ISO 646 und dient} \\

\midrule

\multicolumn{2}{c}{\cellcolor{gray!15}\textbf{Predictions}} \\
\midrule

Ground truth &
\foreignlanguage{ngerman}{als Grundlage für spätere, auf mehr Bits basierende Kodierungen für Zeichensätze. Der ASCII-Code wurde zuerst am 17. Juni 1963 von der American Standards Association (ASA) als Standard ASA X3.4-1963 gebilligt und 1967/1968 wesentlich sowie zuletzt im Jahr 1986 (ANSI X3.4-1986) von ihren Nachfolgeinstitutionen aktualisiert und wird bis heute noch benutzt. Die Zeichenkodierung definiert 128 Zeichen, bestehend aus 33 nicht druckbaren sowie den folgenden 95 druckbaren Zeichen, beginnend mit dem Leerzeichen:} 
{\ttfamily
: !"\#\$ \%\&'()*+,-./0123456789:;=>?@
ABCDEFGHIJKLMNOPQRSTUVWXYZ\textbackslash\^{}\_`abcdefghijklmnopqrstuvwxyz
}
\foreignlanguage{ngerman}{Die druckbaren Zeichen umfassen das lateinische Alphabet in} \\

Vanilla &
\foreignlanguage{ngerman}{dem Zweck der Diskussion über das Unicode Standard . Es ist sicher eine Unicode ( Auctus @-@ Unicode ) mitvernuftiger Anwendung als das Unicode Standard auch @-@ Standards der Deutschen Shell ( ) , über dem Kimia66 von Japan den führenden in Canada fre} \\

Discrete &
\foreignlanguage{ngerman}{jedoch dem erweiterten Ordnungsascor von Microsoft .Auf dem US- Continent ( AUS ) enthält das Standard ( Modulation 3 char ) *00 00 00 00 , also auch @-@ 00 @-@ 00 @-@ 00@ @-@ 00 .Werden durch für uns immer in den USA} \\

NEFTune &
\foreignlanguage{ngerman}{als Codigo de los Estandardales de Información Interchange (AIC; Čeština, alternativ ISO 874) in Modulationle charakteristische Systems. Genauso bekannt ist eine anderen 7-Bit-Zeichenkodierung für Lattes für uns.
• Canada's Universal} \\

Cont. (w/o DB) &
\foreignlanguage{ngerman}{dem US-Standard für Zeichensatz-Dateien (ISO/IEC 9576) 6 Continent Code Auctiec c for shorthand description, (ISO 8859An char ) *Unformatierungstempel
Standardtext (Standardschriftpresseakten.ddd Typen: KimBi, Amersfoort) uns im englischen Alphabet} \\

Cont. (w/ DB) &
\foreignlanguage{ngerman}{dem US-Standard für Zeichensatz-Dateien (ISO/IEC 9576) 6 Continent Code Auctiec c for shorthand description, (ISO 8859An char ) *Unformatierungstempel
Standardtext (Standardschriftpresseakten.ddd Typen: KimBi, Amersfoort) uns im englischen Alphabet} \\

\bottomrule
\end{tabularx}
\end{table}

\section{Additional Discussion}\label{append: additional_discuss}

\subsection{Additional explanations to the score-based method}\label{subsec: additional_expla_method}

In this section, we explain the technical details of our estimation approach in Section~\ref{sec:method}.
For a given parameter $\gamma$, define the score function
for each $t=1,\dots,L$:
\begin{equation*}
s_t(\gamma; X_t, X_{<t}, w) = \frac{\partial}{\partial \gamma} \log P_{\theta} \left( X_t | X_{<t} +  T_\beta(w |X_{<t}) \right) .
\end{equation*}

Given an observed sequence $X_{\leq L}$ and $t$, draw $K$ random samples $\widetilde{w}_t^{(1)}, \dots, \widetilde{w}_t^{(K)}$ that follows distribution $p$ and are independent of $\{w_t^*\}$.
Then for $k=1,\dots,K$, $t=1,\dots,L$, we generate
$\widetilde{X}_t^{(k)} \sim P_{\theta} \big( \bullet | X_{<t} + T_\beta( \widetilde{w}_t^{(k)} |X_{<t})  \big)$.
Further draw $K$ independent samples $w_t^{(1)}, \dots, w_t^{(K)} \sim p$, independently of all other latent variables.
Define the vector-valued function
\begin{equation}
\label{eqn: def_psi}
\begin{split}
    &
    \psi\left( \gamma; X_{\leq L}, \{w_t^{(k)}\}, \{\widetilde{w}_t^{(k)}\} \right) 
    =
    \sum_{k=1}^K \sum_{t=1}^L \left\{ s_t(\gamma; X_t, X_{<t}, w_t^{(k)}) - s_t(\gamma; \widetilde{X}_{t}^{(k)}, X_{<t}, w_t^{(k)}) \right\} .
\end{split}
\end{equation}

Then given a sample of $n$ independent sequences $\{X_{i,\leq L}\}_{i=1}^n$, we construct the estimating equation
\begin{align}
    \Psi(\gamma) = \frac{1}{n} \sum_{i=1}^n \psi \left( \gamma; X_{i,\leq L}, \{w_{i,t}^{(k)}\}, \{\widetilde{w}_{i,t}^{(k)}\} \right) = 0 .
    \label{eqn: def_estimate_eqn}
\end{align}
The above equation can now be solved by SGD.
To this end, note that the term in \eqref{eqn: def_psi} is exactly the gradient (with respect to $\gamma$) of $\ell$ 
(defined in Section~\ref{sec:theory}).
Hence with the definition $\mathcal{L}(\gamma) = \sum_{i=1}^n \ell(\gamma; X_{i,\leq L})$, we have $\Psi(\gamma) = \partial \mathcal{L}(\gamma) / \partial \gamma$. Consequently, solving equation \eqref{eqn: def_estimate_eqn} is equivalent to finding the stationary point of $\mathcal{L}(\gamma)$. The algorithmic procedure is then standard, given that the gradient at each step is the feasible quantity $\Psi(\gamma)$.


\subsection{Additional theoretical results}\label{subsec: additional_thm}

In this section, we give additional theoretical results and discussion, complementing those presented in the main text.
We start by presenting an additional assumption.

\medskip
\begin{assumption}[Variance]\label{ass: variance}
For any $\gamma^* \in \Gamma^*$, denote the matrix $\Sigma(\gamma^*) := \lim_{n\to \infty} \mathrm{Cov}(\sqrt{n} \Psi(\gamma^*))$.
Assume the matrix $\widetilde{\Sigma} := Q^\top \Sigma(\gamma^*) Q$ is positive definite
and does not depend on $\gamma^* \in \Gamma^*$.
\end{assumption}

\medskip
\begin{proposition}[Fisher consistency]\label{prop: fisher_consistency_2}
With the notation in \eqref{eqn: def_estimate_eqn}, 
$\mathbb{E}_{\gamma^*}[\Psi(\gamma^*)] = 0$
for every $\gamma^* \in \Gamma^*$.
\end{proposition}


\medskip
\begin{theorem}[Asymptotic normality]\label{thm: normality_2}
Under Assumptions~\ref{ass: compact}--\ref{ass: variance} and with the notations therein, it holds as $n\to \infty$ that
\begin{align*}
    \sqrt{n} Q^\top (\widehat{\gamma} - g^*)
    =
    Z_{\Gamma^*} 
    + O_P\big( \sqrt{n} \|\Psi(\widehat{\gamma})\| \big)
    + o_P(1) ,
\end{align*}
where $Z_{\Gamma^*} \sim \mathcal{N}\big( 0, (H^*)^{-1} \widetilde{\Sigma} (H^*)^{-1} \big)$.
\end{theorem}

Proposition~\ref{prop: fisher_consistency_2} is a fundamental step to justify our score-based approach in Section~\ref{sec:method}, emphasizing that the solution of our estimating equation \eqref{eqn: def_estimate_eqn} is unbiased.

Theorem~\ref{thm: normality_2} shows the asymptotic normality of our estimator.
In particular, the expansion therein characterizes the convergence of $\sqrt{n} (\widehat{\gamma} - g^*)$ to the classical Gaussian limit $Z_{\gamma^*}$.
Similar to Theorem~\ref{thm: consistency_2}, two sources of errors are explicitly separated:
(i) the $o_P(1)$ term coming from the variance of $\psi$ at the parameter $g^*$;
(ii) the residual term of order $\sqrt{n} \|\Psi(\widehat{\gamma})\|$ that reflects the fact that $\widehat{\gamma}$ only approximately solves the estimating equation (e.g., when optimization is stopped early).
If the estimating equation is solved exactly, i.e.\ $\Psi(\widehat{\gamma}) = 0$, the residual vanishes and we obtain the usual asymptotic normality
\begin{align*}
    \sqrt{n} (\widehat{\gamma} - g^*) \xrightarrow{d}
    \mathcal{N}\big( 0, (H^*)^{-1} \widetilde{\Sigma} (H^*)^{-1} \big) .
\end{align*}
If $\|\Psi(\widehat{\gamma})\|$ is $o_P(n^{-1/2})$, the above remains true according to Theorem~\ref{thm: normality}.

\subsection{Theoretical analysis under identifiablity}
\label{sec:thm_identifiablity}

In addition to the setup considered in the main text, we present here the theoretical guarantee of the parameter estimators under identifiability of the parameter, where there is a unique true parameter such that $\Gamma^* = \{\gamma_0\}$.
With $\psi$ defined in \eqref{eqn: def_psi}, let $\Sigma(\gamma) := \mathbb{E}\big[ \psi(\gamma; \bullet) \psi(\gamma; \bullet)^\top \big]$.

\medskip
\begin{assumption}[Identifiability]\label{ass: identify}
$\mathbb{E}[\psi]$ has a unique zero at $\gamma = \gamma_0$.
\end{assumption}

\medskip
\begin{assumption}[Regularity on $\psi$]\label{ass: reg_psi}
The function $\psi$
is continuous at each $\gamma\in \Gamma$.
\end{assumption}

\medskip
\begin{assumption}[Finite variance]\label{ass: fini_var}
The matrix $\Sigma(\gamma_0)$ is positive definite with eigenvalues bounded away from zero and infinity.
\end{assumption}

\medskip
\begin{assumption}[Smoothness]\label{ass: smooth}
(i) There exists an open neighborhood $\Gamma_0 \subseteq \Gamma$ of $\gamma_0$ such that the function $\psi$ is continuously differentiable on $\Gamma_0$,
and denote the derivative at $\gamma = u$ as $\nabla_{\gamma} \psi(u; \bullet)$;
and (ii) The matrix $V_0 = \mathbb{E}\big[ \nabla_{\gamma} \psi\big( \gamma_0; \bullet \big) \big]$ has full rank.
\end{assumption}

Note that Assumption~\ref{ass: reg_psi} directly holds under Assumption~\ref{ass: smooth}. We state them separately so that we only require a weaker condition for consistency (Theorem~\ref{thm: consistency}).


Let $\widehat{\gamma}$ be some sequence of parameter estimator that asymptotically solves \eqref{eqn: def_estimate_eqn}, as $n\to \infty$.

\medskip
\begin{proposition}[Fisher consistency]\label{prop: fisher_consistency}
With the notation in \eqref{eqn: def_estimate_eqn}, we have $\mathbb{E}_{\gamma_0}[\Psi(\gamma_0)] =0$.
\end{proposition}

\medskip
\begin{theorem}[Consistency]\label{thm: consistency}
Under Assumptions~\ref{ass: compact}, \ref{ass: identify}, \ref{ass: reg_psi},
$\widehat{\gamma}$ converges to $\gamma_0$ in probability, as $n\to \infty$.
\end{theorem}

\medskip
\begin{theorem}[Asymptotic normality]\label{thm: normality}
Under Assumptions~\ref{ass: compact}, \ref{ass: identify}, \ref{ass: reg_psi}, \ref{ass: fini_var} and \ref{ass: smooth} and with the notations therein, it holds as $n\to \infty$ that
\begin{align*}
    \sqrt{n} (\widehat{\gamma} - \gamma_0)
    =
    Z_{\gamma_0} + O_P( \sqrt{n}  \|\Psi(\widehat{\gamma}) \|) + o_P(1) ,
\end{align*}
where $Z_{\gamma_0} \sim \mathcal{N}\big( 0, V_0^{-1} \Sigma(\gamma_0) (V_0^{-1})^\top \big)$.
\end{theorem}

The results are analogous to the scenario (with more than one parameter in $\Gamma^*$) shown in the main text.

\section{Auxiliary Results and All the Proofs}\label{append: proof}

\subsection{Lemmas with proofs}

\begin{lemma}[Uniqueness of the identifiable component]\label{lem: unique_iden}
Under Assumption~\ref{ass: Hessian}, we have:
\begin{itemize}
    \item [(i)] The set $Q^\top \Gamma^* := \{Q^\top \gamma^*: \gamma^* \in \Gamma^*\}$ is a singleton.
    \item [(ii)] For any $\gamma \in\Gamma$, denote the projection of $\gamma$ onto the convex set $\Gamma^*$ as
    \[
    \Pi_{\Gamma^*}(\gamma) := \arg\min_{\gamma^* \in \Gamma^*} \|\gamma - \gamma^*\| .
    \]
    Then the distance $d(\gamma, \Gamma^*)$ can be written as 
    $\|Q^\top (\gamma - \Pi_{\Gamma^*}(\gamma))\|$.
\end{itemize}
\end{lemma}

\begin{proof}[Proof of Lemma~\ref{lem: unique_iden}]
Consider part~(i).
Fix any $\gamma^*\in \Gamma^*$, and recall $H^* = Q^\top H(\gamma^*) Q$ with $Q$ from Assumption~\ref{ass: Hessian}~(iii).
Note that $H^*$ is strictly negative definite.
By a straightforward argument using singular value decomposition (cf.\ proof of Lemma~15 in \cite{Zhouetal2026}), we equivalently have $Q^\top H(\gamma^*) = H^* Q^\top$.
Take any $\gamma_1^*, \gamma_2^2 \in \Gamma^*$, by the convexity of $\Gamma^*$ from Assumption~\ref{ass: Hessian}~(i), the segment $\gamma(\tau) = \gamma_1^* + \tau(\gamma_2^* -\gamma_1^*)$ lies in $\Gamma^*$ for all $\tau \in[0,1]$.
Together with the definition that $\Gamma^* = \{\gamma\in \Gamma: \nabla J(\gamma) =0\}$, we have $\nabla J(\gamma(\tau)) =0$ for all $\tau\in[0,1]$.
Then we have
\begin{align*}
    0 &= Q^\top \cdot 0
    =
    Q^\top \cdot \frac{\mathrm{d}}{\mathrm{d}\tau} \nabla J(\gamma(\tau)) \\
    &=
    Q^\top \cdot H(\gamma(\tau)) \cdot (\gamma_2^* -\gamma_1^*) 
    =
    H^* Q^\top (\gamma_2^* -\gamma_1^*) ,
\end{align*}
thus indicating $Q^\top \gamma_2^* = Q^\top \gamma_1^*$, since $H^*$ is invertible.
This concludes the proof for part~(i).

Next, for (ii), by definitions we have
\begin{align*}
    d(\gamma, \Gamma^*)
    = \inf_{\gamma^*\in \Gamma^*} \|\widehat{\gamma} - \gamma^*\|
    = \inf_{\gamma^*\in \Gamma^*} \|\widehat{\gamma} - \Pi_{\Gamma^*}(\gamma)\|
    = \|Q^\top (\gamma - \Pi_{\Gamma^*}(\gamma))\| ,
\end{align*}
since $Q^\top Q = I$.
This completes the proof of this lemma.
\end{proof}

\subsection{Proofs of theorems and propositions}

\begin{proof}[Proof of Proposition~\ref{prop: fisher_consistency_2}]
This is straightforward from the definition of $\Gamma^*$.
It is worthwhile to also refer to the proof of Proposition~\ref{prop: fisher_consistency}, for any $\gamma^* \in \Gamma^*$, so as to see our definition of $\Gamma^*$ is reasonable.
\end{proof}

\medskip
\begin{proof}[Proof of Proposition~\ref{prop: fisher_consistency}]

By linearity of expectation, it suffices to show the expectation is zero for any fixed $i=1,\dots, n$, $k=1,\dots, K$, and $t=1,\dots, L$. Namely, using notations in \eqref{eqn: def_psi}, we want to show
\begin{align*}
    \mathbb{E}_{\gamma_0}\left[ s_t(\gamma_0; X_t, X_{<t}, w_t^{(k)}) - s_t(\gamma_0; \widetilde{X}_{t}^{(k)}, X_{<t}, w_t^{(k)}) \right] =0 .
\end{align*}

Consider first $ s_t(\gamma_0; X_t, X_{<t}, w_t^{(k)})$. 
From the likelihood of $X_t$ in Section~\ref{sec:method},
\begin{align}
    &\;\quad
    \mathbb{E}_{\gamma_0}\left[ s_t(\gamma_0; X_t, X_{<t}, w_t^{(k)}) \,\big|\, X_{<t} , w_t^{(k)}, w_t^* \right] \notag \\
    &=
    \int s_t(\gamma_0; x, X_{<t}, w_t^{(k)}) \, P_{\theta_0} \left( x \,|\, X_{<t} +  T_{\beta_0}(w_t^* |X_{<t})  \right) \,\mathrm{d}x .
    \label{eqn: expect_gamma0_w_ast}
\end{align}
Similarly, for $s_t(\gamma_0; \widetilde{X}_{t}^{(k)}, X_{<t}, w_t^{(k)})$ we have
\begin{align}
    &\;\quad
    \mathbb{E}_{\gamma_0}\left[ s_t(\gamma_0; \widetilde{X}_{t}^{(k)}, X_{<t}, w_t^{(k)}) \,\big|\, X_{<t} , w_t^{(k)}, \widetilde{w}_t^{(k)} \right] \notag \\
    &=
    \int s_t(\gamma_0; x, X_{<t}, w_t^{(k)}) \, P_{\theta_0} \left( x \,|\, X_{<t} +  T_{\beta_0}(\widetilde{w}_t^{(k)} |X_{<t})  \right) \,\mathrm{d}x .
    \label{eqn: expect_gamma0_w_tilde}
\end{align}

Note that both $w_t^*$ and $\widetilde{w}_t^{(k)}$ are distributed as $p$, and they are independent of each other. Together with the law of iterated expectation, the unconditional distribution (with respect to the true parameters $\gamma_0$) of \eqref{eqn: expect_gamma0_w_ast} and \eqref{eqn: expect_gamma0_w_tilde} are the same and hence their difference is 0.
This completes the proof.
\end{proof}

\medskip
\begin{proof}[Proof of Theorem~\ref{thm: consistency}]
First, by Assumptions~\ref{ass: compact} and \ref{ass: reg_psi},
the function $\psi$ is bounded for all $\gamma\in \Gamma$ by the extreme value theorem, and hence we may apply Lemma~2.4 in \cite{NeweyMcFadden1994} and obtain that $\mathbb{E}[\Psi(\gamma)]$ is continuous for every $\gamma \in \Gamma$, and as sample size $n\to \infty$:
\begin{align}
    \sup_{\gamma\in \Gamma} \big\| \Psi(\gamma) - \mathbb{E}[\Psi(\gamma)] \big\| \xrightarrow[]{p} 0 .
    \label{eqn: unif_LLN}
\end{align}

Let $\epsilon >0$ be any arbitrary constant as given.
By the continuity of $\mathbb{E}[\Psi(\gamma)]$ and Assumption~\ref{ass: identify}, it holds that
\begin{align}
    \inf_{\gamma: \|\gamma - \gamma_0\| \geq \epsilon} \big\| \mathbb{E}[\Psi(\gamma)] \big\| \geq \eta > 0 ,
    \label{eqn: well_pose_expect_Psi}
\end{align}
for some absolute constant $\eta$. 
Define the event
\begin{align*}
    \mathcal{E} := \left\{ \sup_{\gamma\in \Gamma} \big\| \Psi(\gamma) - \mathbb{E}[\Psi(\gamma)] \big\| < \frac{\eta}{2} \right\} .
\end{align*}
On $\mathcal{E}$, for any $\gamma$ such that $\|\gamma - \gamma_0\| \geq \epsilon$, we have
\begin{align*}
    \big\| \Psi(\gamma) \big \| \geq
    \big\| \mathbb{E}[\Psi(\gamma)] \big\| 
    - \big\| \Psi(\gamma) - \mathbb{E}[\Psi(\gamma)] \big\|
    \geq
    \eta - \frac{\eta}{2} = \frac{\eta}{2} > 0 ,
\end{align*}
which, together with $\|\Psi(\widehat{\gamma})\| \leq \eta/2$ for large enough $n$, indicates that
$\|\widehat{\gamma} - \gamma_0\| <\epsilon$ on the set $\mathcal{E}$.

Finally, \eqref{eqn: unif_LLN} implies that as $n\to \infty$ we have $\mathbb{P}(\mathcal{E}) \to 1$, and hence
\begin{align*}
    \mathbb{P}\left( \big\| \widehat{\gamma} - \gamma_0 \big\| < \epsilon \right)
    \geq \mathbb{P}(\mathcal{E})
    \to 1 ,
\end{align*}
as desired.
This completes the proof.
\end{proof}

\medskip
\begin{proof}[Proof of Theorem~\ref{thm: normality}]

First, note that $\psi$ defined \eqref{eqn: def_psi} has zero mean at $\gamma=\gamma_0$ according to Proposition~\ref{prop: fisher_consistency}, with covariance matrix $\Sigma(\gamma_0)$ from Assumption~\ref{ass: fini_var}. Since every $\psi$ in \eqref{eqn: def_estimate_eqn} is i.i.d.\ over $i=1,\dots, n$, by central limit theorem (cf.\ Example~2.18 of \cite{Vaart1998}) we have
\begin{align}
    \sqrt{n} \Psi(\gamma_0) =
    \frac{1}{\sqrt{n}} \sum_{i=1}^n \psi \left( \gamma; X_{i,\leq L}, \{w_{i,t}^{(k)}\}, \{\widetilde{w}_{i,t}^{(k)}\} \right)
    \xrightarrow[]{d}
    \mathcal{N}(0, \Sigma(\gamma_0)) .
    \label{eqn: asymp_gauss_Psi_gamma0}
\end{align}

Since $\widehat{\gamma} \xrightarrow[]{p} \gamma_0$ according to Theorem~\ref{thm: consistency}, it holds with probability approaching 1 that $\widehat{\gamma} \in \Gamma_0$. Hence by Assumption~\ref{ass: smooth}~(i), we may expand $\Psi(\widehat{\gamma})$ around $\gamma_0$ using the integral form of Taylor theorem, such that
\begin{align}
    \Psi(\widehat{\gamma}) &= \Psi(\gamma_0) + 
    \left\{ \int_{0}^{1} \frac{1}{n} \sum_{i=1}^n \nabla_{\gamma} \psi\big( \gamma_0 + t(\widehat{\gamma} - \gamma_0); X_{i,\leq L}, \{w_{i,t}^{(k)}\}, \{\widetilde{w}_{i,t}^{(k)}\} \big) \,\mathrm{d}t \right\} \cdot (\widehat{\gamma} - \gamma_0) \notag \\
    &=: 
    \Psi(\gamma_0) + 
    \overline{V} \cdot (\widehat{\gamma} - \gamma_0)  .
    \label{eqn: Psi_gamma_hat_expan}
\end{align}
The above is hence valid holds for any closed set $\overline{\Gamma}_0 \subset \Gamma_0$. 
Note that $\overline{\Gamma}_0$ is thus compact by Assumption~\ref{ass: compact}.
Also, by Assumption~\ref{ass: smooth}~(i), $\nabla_{\gamma} \psi$ is continuous and bounded (by the extreme value theorem) in $\Gamma_0$. Altogether, by the similar argument for \eqref{eqn: unif_LLN}, for all $\gamma \in \overline{\Gamma}_0$ we have
\begin{align}
    \sup_{\gamma\in \overline{\Gamma}_0} \Big\| \nabla_{\gamma} \psi\big( \gamma; \bullet \big)  - \mathbb{E}\big[ \nabla_{\gamma} \psi\big( \gamma; \bullet \big) \big] \Big\| \xrightarrow[]{p} 0 .
    \label{eqn: unif_LLN_grad}
\end{align}

Further note that the mapping $\gamma: \mapsto \int_{0}^{1} \nabla_{\gamma} \psi\big( \gamma + t(\widehat{\gamma} - \gamma); \bullet \big) \,\mathrm{d}t$ is continuous, by the fundamental theorem of calculus and the fact that $\nabla_{\gamma} \psi$ is integrable (due to boundedness).
Then by \eqref{eqn: unif_LLN_grad} and the continuous mapping theorem, $\overline{V}$ from \eqref{eqn: Psi_gamma_hat_expan} satisfies
\begin{align}
    \overline{V} &\xrightarrow[]{p} 
    \int_0^1 \frac{1}{n} \sum_{i=1}^n \mathbb{E}\big[ \nabla_{\gamma} \psi\big( \gamma_0 + t(\widehat{\gamma} - \gamma_0); \bullet \big) \big] \,\mathrm{d}t \\
    &\xrightarrow[]{p} 
    \int_0^1 \frac{1}{n} \sum_{i=1}^n \mathbb{E}\big[ \nabla_{\gamma} \psi\big( \gamma_0; \bullet \big) \big] \,\mathrm{d}t 
    =
    \mathbb{E}\big[ \nabla_{\gamma} \psi\big( \gamma_0; \bullet \big) \big] 
    = V_0,
    \label{eqn: V_overline_converge}
\end{align}
where the second convergence used $\widehat{\gamma} \xrightarrow[]{p} \gamma_0$ from Theorem~\ref{thm: consistency} and the continuous mapping theorem again.
Since $V_0$ is invertible by Assumption~\ref{ass: smooth}~(ii) with probability approaching 1, by \eqref{eqn: asymp_gauss_Psi_gamma0} and \eqref{eqn: Psi_gamma_hat_expan}, we conclude that
\begin{align*}
    \sqrt{n} (\widehat{\gamma} - \gamma_0)
    \xrightarrow[]{p}
    \sqrt{n} V_0^{-1} \Psi(\widehat{\gamma})
    - \sqrt{n} V_0^{-1} \Psi(\gamma_0) 
    =
    Z_{\gamma_0} + \sqrt{n} V_0^{-1} \Psi(\widehat{\gamma}) + o_P(1) ,
\end{align*}
with
\begin{align*}
    Z_{\gamma_0} \sim 
    \mathcal{N}\big( 0, V_0^{-1} \Sigma(\gamma_0) (V_0^{-1})^\top \big) .
\end{align*}
This completes the proof.
\end{proof}

\medskip
\begin{proof}[Proof of Theorem~\ref{thm: consistency_2}]
Note that $\nabla J(\gamma) = \mathbb{E}[\Psi(\gamma)]$ from definitions.
From Assumption~\ref{ass: smooth_J}, $\psi$ is continuous. Then using exactly the same argument for \eqref{eqn: unif_LLN}, the uniform convergence also holds here, i.e.,
as $n\to \infty$:
\begin{align}
    \sup_{\gamma\in \Gamma} \big\| \Psi(\gamma) - \nabla J(\gamma) \big\| \xrightarrow[]{p} 0 .
    \label{eqn: unif_LLN_J}
\end{align}

Let $\epsilon >0$ be any arbitrary constant as given.
Define the event 
\begin{align*}
    \mathcal{S}_{\epsilon} := \big\{ \gamma\in \Gamma: d(\gamma, \Gamma^*) \geq \epsilon \big\} .
\end{align*}
Since $\Gamma^* = \{\gamma\in \Gamma: \nabla J(\gamma) =0\}$, any $\gamma \in \mathcal{S}_{\epsilon}$ is not a stationary point.
Hence $c_{\epsilon} := \inf_{\gamma \in \mathcal{S}_{\epsilon}} \|\nabla J(\gamma)\| \geq \eta > 0$ for some absolute constant $\eta$. This plays the same role of \eqref{eqn: well_pose_expect_Psi} in the proof of Theorem~\ref{thm: consistency_2},
and together with \eqref{eqn: unif_LLN_J}, the remaining arguments are repetitive as in the proof of Theorem~\ref{thm: consistency_2}.
This shows that as $n\to \infty$, $d(\widehat{\gamma}, \Gamma^*) \xrightarrow{p} 0$,
as desired.
\end{proof}

\medskip
\begin{proof}[Proof of Theorem~\ref{thm: consistency_2_rate}]
Given $\widehat{\gamma}$, we use the notation from Lemma~\ref{lem: unique_iden}~(ii) and denote $\Pi_{\Gamma^*}(\widehat{\gamma})$ the projection of $\widehat{\gamma}$ onto $\Gamma^*$.
By Assumption~\ref{ass: smooth_J}, we may expand $\Psi(\widehat{\gamma})$ around a point in $\Gamma_0$.
By the consistency result shown earlier, we may expand around $\Pi_{\Gamma^*}(\widehat{\gamma})$ using the integral form of Taylor theorem, such that
\begin{align}
    &\;\quad
    \Psi(\widehat{\gamma}) -
    \Psi(\Pi_{\Gamma^*}(\widehat{\gamma}))
    \notag \\
    &=
    \left\{ \int_{0}^{1} \frac{1}{n} \sum_{i=1}^n \nabla_{\gamma} \psi\big( \Pi_{\Gamma^*}(\widehat{\gamma}) + t(\widehat{\gamma} - \Pi_{\Gamma^*}(\widehat{\gamma})); X_{i,\leq L}, \{w_{i,t}^{(k)}\}, \{\widetilde{w}_{i,t}^{(k)}\} \big) \,\mathrm{d}t \right\} \cdot  (\widehat{\gamma} - \Pi_{\Gamma^*}(\widehat{\gamma})) \notag \\
    &=: 
    \overline{H} \cdot (\widehat{\gamma} - \Pi_{\Gamma^*}(\widehat{\gamma}))  ,
    \label{eqn: Psi_gamma_hat_expan_2}
\end{align}
which also holds for any closed set $\overline{\Gamma}_0 \subset \Gamma_0$. 
Note that $\overline{\Gamma}_0$ is thus compact by Assumption~\ref{ass: compact}.
By Assumptions~\ref{ass: compact} and \ref{ass: smooth_J}, the function class $\{\nabla_\gamma \psi(\gamma; \bullet) : \gamma\in \overline{\Gamma}_0\}$ is Donsker, and hence
\begin{align}
    \sup_{\gamma\in \overline{\Gamma}_0} \Big\| \nabla_{\gamma} \psi\big( \gamma; \bullet \big)  - H(\gamma) \Big\| \xrightarrow[]{p} 0 .
    \label{eqn: unif_LLN_grad_2}
\end{align}
Hence $\overline{H} = H(\Pi_{\Gamma^*}(\widehat{\gamma})) + E_H$, with $\|E_H\| = o_P(1)$, by analogous arguments as in \eqref{eqn: V_overline_converge}.
Plugging this in \eqref{eqn: Psi_gamma_hat_expan_2} and left-multiplying $Q^\top$, we have
\begin{align}
    Q^\top \Psi(\widehat{\gamma})
    &= Q^\top \Psi(\Pi_{\Gamma^*}(\widehat{\gamma}))
    + Q^\top H(\Pi_{\Gamma^*}(\widehat{\gamma})) (\widehat{\gamma} - \Pi_{\Gamma^*}(\widehat{\gamma})) 
    + Q^\top E_H (\widehat{\gamma} - \Pi_{\Gamma^*}(\widehat{\gamma})) \notag \\
    &=
    Q^\top \Psi(\Pi_{\Gamma^*}(\widehat{\gamma}))
    + Q^\top H(\Pi_{\Gamma^*}(\widehat{\gamma})) (\widehat{\gamma} - \Pi_{\Gamma^*}(\widehat{\gamma}))
    + o_P\big( \|\widehat{\gamma} - \Pi_{\Gamma^*}(\widehat{\gamma})\| \big) .
    \label{eqn: Q_Psi_gamma_hat_decomp}
\end{align}

Note that Assumption~\ref{ass: compact} and \ref{ass: smooth_J} also imply the functions class $\{\psi(\gamma; \bullet): \gamma\in \overline{\Gamma}_0\}$ is Donsker and thus, together with \eqref{eqn: unif_LLN_J}, gives $\sup_{\gamma\in \overline{\Gamma}_0} \big\| \Psi(\gamma) - \nabla J(\gamma) \big\| = O_P(n^{-1/2})$.
In particular, $\Pi_{\Gamma^*}(\widehat{\gamma}) \in \Gamma^*$ and hence $\nabla J(\Pi_{\Gamma^*}(\widehat{\gamma})) = 0$ by definition, so that we have
\begin{align}
    \big\| \Psi(\Pi_{\Gamma^*}(\widehat{\gamma})) \big\| = O_P(n^{-1/2}) .
    \label{eqn: Psi_Pi_gamma_hat_rate}
\end{align}

Finally, recall $H^* = Q^\top H(\Pi_{\Gamma^*}(\widehat{\gamma})) Q$ with $Q$ from Assumption~\ref{ass: Hessian}~(iii).
Note that $H^*$ is strictly negative definite and hence invertible. This implies there exists some constant $c_{H^*}$ such that $\|H^* u\| \geq c_{H^*} \|u\|$ for any vector $u$ (with compatible dimension).
Then from Lemma~\ref{lem: unique_iden}~(ii), we have
\begin{align*}
    c_{H^*} \cdot
    d(\widehat{\gamma}, \Gamma^*) &= c_{H^*} \big\| Q^\top (\widehat{\gamma} - \Pi_{\Gamma^*}(\widehat{\gamma})) \big\|
    \leq
    \big\| H^* Q^\top (\widehat{\gamma} - \Pi_{\Gamma^*}(\widehat{\gamma})) \big\| \\
    &=
    \big\| Q^\top H(\Pi_{\Gamma^*}(\widehat{\gamma}))  (\widehat{\gamma} - \Pi_{\Gamma^*}(\widehat{\gamma})) \big\| \\
    &\leq
    O_P(n^{-1/2}) + \|\Psi(\widehat{\gamma})\|
    + o_P\big( \|\widehat{\gamma} - \Pi_{\Gamma^*}(\widehat{\gamma})\| \big)
\end{align*}
where the second equality used the same argument in the proof of Lemma~\ref{lem: unique_iden}~(i),
and the second inequality used \eqref{eqn: Q_Psi_gamma_hat_decomp} and \eqref{eqn: Psi_Pi_gamma_hat_rate}.
The result is as desired by noting that
$c_{H^*}$ is a constant,
and $\|\widehat{\gamma} - \Pi_{\Gamma^*}(\widehat{\gamma})\| = \|Q^\top (\widehat{\gamma} - \Pi_{\Gamma^*}(\widehat{\gamma}))\| = d(\widehat{\gamma}, \Gamma^*)$ by Lemma~\ref{lem: unique_iden}~(ii).
This concludes the proof.
\end{proof}

\begin{proof}[Proof of Theorem~\ref{thm: normality_2}]
From the proof of Theorem~\ref{thm: consistency_2}, we can write \eqref{eqn: Q_Psi_gamma_hat_decomp} as
\begin{align*}
    H^* Q^\top (\widehat{\gamma} - \Pi_{\Gamma^*}(\widehat{\gamma}))
    =
    Q^\top \Psi(\Pi_{\Gamma^*}(\widehat{\gamma}))
    + O_P\big( \|\Psi(\widehat{\gamma})\| \big)
    + o_P\big( n^{-1/2} \big) ,
\end{align*}
where $o_P\big( n^{-1/2} \big)$ term is from the rate of convergence by Theorem~\ref{thm: consistency_2}.
Multiplying by $\sqrt{n}$ on both sides above, we have
\begin{align*}
    \sqrt{n}  H^* Q^\top (\widehat{\gamma} - \Pi_{\Gamma^*}(\widehat{\gamma}))
    =
    \sqrt{n}  Q^\top \Psi(\Pi_{\Gamma^*}(\widehat{\gamma}))
    + O_P\big( \sqrt{n} \|\Psi(\widehat{\gamma})\| \big)
    + o_P\big(1 \big) .
\end{align*}
Since $H^*$ is invertible, it suffices to show the limiting distribution for $\sqrt{n}  Q^\top \Psi(\Pi_{\Gamma^*}(\widehat{\gamma}))$ and the remaining arguments are similar to the proof of Theorem~\ref{thm: normality}.

Consider $\sqrt{n}  Q^\top \Psi(\Pi_{\Gamma^*}(\widehat{\gamma}))$. Note that $\Pi_{\Gamma^*}(\widehat{\gamma}) \in \Gamma^*$ and hence $\nabla J(\Pi_{\Gamma^*}(\widehat{\gamma})) =0$.
Then we can write $\sqrt{n}  Q^\top \Psi(\Pi_{\Gamma^*}(\widehat{\gamma})) = Q^\top \mathbb{G}_n(\Pi_{\Gamma^*}(\widehat{\gamma}))$ where $\mathbb{G}_n(\gamma) := \sqrt{n} [\Psi(\gamma) - \nabla J(\gamma) ]$ is the empirical process.
It is worth pointing out that $\mathbb{G}_n(\Pi_{\Gamma^*}(\widehat{\gamma}))$ involves the projection $\Pi_{\Gamma^*}(\widehat{\gamma})$ which is not unique, despite of the uniquely identifiable vector after projection by $Q^\top$, according to Lemma~\ref{lem: unique_iden}~(i).

Recall from the proof of Theorem~\ref{thm: consistency_2} that the functions class $\{\psi(\gamma; \bullet): \gamma\in \overline{\Gamma}_0\}$ is Donsker.
Hence we have stochastic equicontinuity (see e.g.\ Section~2.8.3 of \cite{VaartWellner2023}) such that for any $\epsilon>0$:
\begin{align*}
    \lim_{\delta \to 0} \limsup_{n\to \infty} 
    \mathbb{P}^*
    \bigg( \sup_{\|\gamma_1 - \gamma_2\| < \delta} \| \mathbb{G}_n(\gamma_1) - \mathbb{G}_n(\gamma_2) \| > \epsilon \bigg) = 0 .
\end{align*}
where $\mathbb{P}^*$ denotes the outer probability measure. 

Note that $\Gamma^*$ is a closed subset (Assumption~\ref{ass: Hessian}~(i)) of the compact $\Gamma$ (Assumption~\ref{ass: compact}) and is hence compact.
Therefore, we can cover $\Gamma^*$ by finitely many open balls $B_1, \dots, B_m$, centered at $\gamma_{(1)}, \dots, \gamma_{(m)}$, respectively.

For each sample size $n$, the resulting estimator $\widehat{\gamma}$ (corresponding to each particular $n$) has its projection $\Pi_{\Gamma^*}(\widehat{\gamma})$ in $\Gamma^*$, and there exists a random index $j_n \in \{1,\dots, m\}$ such that $\|\Pi_{\Gamma^*}(\widehat{\gamma}) - \gamma_{(j_n)}\| < \delta$.
For any $\delta >0$,
\begin{align}
    \big\| \mathbb{G}_n(\Pi_{\Gamma^*}(\widehat{\gamma})) - \mathbb{G}_n(\gamma_{(j_n)}) \big\|
    \leq
    \sup_{\gamma\in B_{j_n}}
    \big\| \mathbb{G}_n(\gamma) - \mathbb{G}_n(\gamma_{(j_n)}) \big\|
    \leq
    \sup_{\|\gamma_1 - \gamma_2\| < \delta} \big\| \mathbb{G}_n(\gamma_1) - \mathbb{G}_n(\gamma_2) \big\| .
    \label{eqn: delta_GG_minus}
\end{align}
Now from the equicontinuity, fix any $\epsilon>0$ and there exists $\delta>0$ and $N$ such that for all $n\geq N$:
\begin{align*}
    \mathbb{P}^* \bigg( \sup_{\|\gamma_1 - \gamma_2\| < \delta} \| \mathbb{G}_n(\gamma_1) - \mathbb{G}_n(\gamma_2) \| > \epsilon \bigg) < \epsilon ,
\end{align*}
Hence with $\delta$ in \eqref{eqn: delta_GG_minus} set as such $\delta$ above, it holds with (outer) probability $1-\epsilon$ that
\begin{align*}
    \big\| \mathbb{G}_n(\Pi_{\Gamma^*}(\widehat{\gamma})) - \mathbb{G}_n(\gamma_{(j_n)}) \big\|
    \leq
    \epsilon .
\end{align*}
As $m$ is finite, we conclude that $\big\| \mathbb{G}_n(\Pi_{\Gamma^*}(\widehat{\gamma})) - \mathbb{G}_n(\gamma_{(j_n)}) \big\| = o_P(1)$ and thus
\begin{align}
    \left\| Q^\top \left[ \mathbb{G}_n(\Pi_{\Gamma^*}(\widehat{\gamma})) - \mathbb{G}_n(\gamma_{(j_n)}) \right] \right\| = o_P(1) .
    \label{eqn: equicontinuity_consequence}
\end{align}

For any fixed $\gamma^* \in \Gamma^*$, the argument for \eqref{eqn: asymp_gauss_Psi_gamma0} holds and we have
\begin{align*}
    \sqrt{n} \Psi(\gamma^*) 
    = \sqrt{n} (\Psi(\gamma^*) - \nabla J(\gamma^*)) \xrightarrow{d}
    \mathcal{N}(0, \Sigma(\gamma^*)) ,
\end{align*}
and hence $Q^\top \mathbb{G}_n(\gamma^*) =  Q^\top \sqrt{n} \Psi(\gamma^*) \xrightarrow{d} \mathcal{N}(0, \widetilde{\Sigma})$ by the notations from Assumption~\ref{ass: variance}.

Finally, we will leverage the fact that if every subsequence of a sequence contains a further subsequence converging to the same limit, then the original sequence converges to that limit.
To this end, note that for any subsequence $\{n_k\}$ of the sequence $\{1,2,\dots\}$, there exists a subsequence $\{n_{k_l}\}$ along which the index $j_{n_{k_l}}$ is constant, denoted as some index $j_0$.
Along such subsequence $\{n_{k_l}\}$, $\Pi_{\Gamma^*}(\widehat{\gamma})$ is a fixed vector in $\Gamma^*$ (i.e., the ball center of $B_{j_0}$).
Together with \eqref{eqn: equicontinuity_consequence}, we conclude that
\begin{align*}
    Q^\top \mathbb{G}_{n_k}(\Pi_{\Gamma^*}(\widehat{\gamma}))
    \xrightarrow{d} \mathcal{N}(0, \widetilde{\Sigma}) ,
\end{align*}
which is the same limiting distribution for any arbitrary subsequence $\{n_k\}$.
This concludes the proof.
\end{proof}

\section{Details of Experiments}

\subsection{Details on simulation studies}\label{appsec:sim_details}

\textbf{Data generation process}.
we set $\pi_0$ to be uniform over \(\mathcal{V}\), and each
row of \(M_0\) is independently drawn from a Dirichlet distribution with concentration parameter \(0.5\). 
Let \(O_t\in\mathcal{V}\) denote the observed token at time \(t\), and
let \(X_t\in \mathbb{R}^d\) denote the embedding representation associated with $O_t$. 
We first construct a ground-truth autoregressive model $P^*$ such that, in the absence of perturbation, $P^*(\bullet | X_{t-1})=M_0(O_{t-1},\bullet)=\mathbb{P}(O_t=\bullet|O_{t-1}),$ where the conditional distribution depends only on the previous token. 
At each time step, we sample an unobserved latent variable
\(w_t^*\sim \mathcal{N}(0,I_r)\)
and perturb the embedding of the previous token through a fixed neural mapping \(T_{\beta_0}\), so as to construct
\(\widetilde X_{t-1}=X_{t-1}+\alpha T_{\beta_0}(w_t^*| X_{t-1}),\)
where \(\alpha\ge 0\) controls the perturbation strength. The next
token is then generated according to
\(O_t\sim P^* (\bullet | \widetilde X_{t-1})\).
The ground-truth perturbation transformation is implemented
as a three-layer feedforward neural network with ReLU activations. 
We set the latent perturbation dimension $r=8$, the embedding dimension $d=50$, and the hidden layer width $h=64$. Given a latent perturbation vector $w_t^*\in\mathbb{R}^r$ and the embedding representation $X_{t-1}\in\mathbb{R}^d$, the perturbation transformation is defined as 
\[T_{\beta_0}(w_t^\ast \mid X_{t-1})=\bm{W}_3^{(0)}\phi\!\left(\bm{W}_2^{(0)}\phi\!\left(\bm{W}_1^{(0)}
\begin{bmatrix}
w_t^* \\
X_{t-1}
\end{bmatrix}
+\bm{b}_1^{(0)}\right)+\bm{b}_2^{(0)}\right)+\bm{b}_3^{(0)},\]
where
$\phi(x)=\max(x,0)$
denotes the ReLU activation function. 
The collection of all network parameters is denoted by
\(
\beta_0=\left\{\bm{W}_1^{(0)},\bm{W}_2^{(0)},\bm{W}_3^{(0},\bm{b}_1^{(0)},\bm{b}_2^{(0)},\bm{b}_3^{(0)}\right\}\).
Although the perturbation acts in the embedding space, the resulting
transition distribution remains well-defined on the observed token space by marginalizing
over the latent perturbation:
\[
M_\alpha(u,v)=\mathbb{P}(O_t=v| O_{t-1}=u)
=\mathbb{E}_{w}\left[P_{\theta_0}\left(v| X_u+\alpha T_{\beta_0}(w| X_u)\right) \right].
\]
When \(\alpha=0\), we have \(M_\alpha=M_0\), reducing to the standard bigram model.

\textbf{Architecture of $P_{\theta}$}. 
The base autoregressive model \(P_\theta\) is implemented as a neural bigram language model.
Specifically, given the embedding representation $X_{t-1}$ of the previous token $O_{t-1}$, the conditional distribution of the next token is modeled by 
$$P_{\theta}(O_t|X_{t-1})=\text{SoftMax}\left(\bm{W}_2^P \text{Dropout}\left(\text{ReLU}\left(\bm{W}_1^P X_{t-1}+\bm{b}_1^P\right)\right)+\bm{b}_2^P\right),$$
where the trainable parameters $\theta$ consist of $\bm{W}_1^P, \bm{W}_2^P, \bm{b}_1^P$, and $\bm{b}_2^P$. 
The embedding dimension and the output of the first layer are fixed as $d=50$, and the dropout rate is set to $0.1$.

\textbf{Architecture of $T_{\beta}$}.
We use a parameterized perturbation mapping \(T_{\beta}\) to estimate the latent perturbation process. 
We set the latent perturbation dimension $r=8$, the embedding dimension $d=50$, and the hidden dimension of the recurrent encoder $h=64$.
Given an embedding sequence \(X_{<t}\), we first compute contextual
representations using a recurrent encoder:
\((h_1,\dots,h_{t-1})=\mathrm{Enc}_\beta(X_{<t})\), where \(\mathrm{Enc}_\beta(\cdot)\) is implemented as a single-layer LSTM. The pooled context representation is then obtained via mean pooling
\(c_\beta(X_{<t})=\sum_{i=1}^{t-1} h_i/(t-1)\). 
Given a latent perturbation vector
\(w\sim\mathcal N(0,I_r),\) the perturbation generator produces
\[u_\beta(w,X_{<t})=\bm W_2^\beta\phi\!\left(\bm W_1^\beta
\begin{bmatrix}
w\\
c_\beta(X_{<t})
\end{bmatrix}
+\bm b_1^\beta\right)+\bm b_2^\beta.
\]
The output vector is then reshaped into a perturbation matrix
\(W_{<t}=\mathrm{reshape}\{u_\beta(w,X_{<t})\}.\)
The trainable parameter set \(\beta\) consists of both the recurrent encoder parameters and the parameters of the perturbation generator parameters.

\textbf{Training configuration}.
All models are trained using the Adam optimizer \citep{kingma2015adam}.
For the proposed continuous perturbation framework, we adopt separate learning
rates for the base autoregressive model and the perturbation module, set to
\(10^{-2}\) and \(5\times10^{-5}\), respectively, and use $K=5$ perturbation samples during training. For the discrete perturbation baselines, we use a learning rate of \(5\times10^{-3}\) together with a weight decay parameter of \(10^{-4}\).
Across all experiments, the batch size is fixed at \(500\), and optimization
is performed for \(25\) epochs. To ensure a fair comparison between continuous and discrete perturbation methods, we additionally duplicate the original training corpus once when training the continuous perturbation methods, since discrete perturbed training procedure augments the original dataset with an additional perturbed version of the samples.

\subsection{Details on real-world language modeling}\label{appsec:real_details}

\textbf{Source of datasets}. 
The WikiText-2 and WikiText-103 corpora are obtained from the Hugging Face dataset repository \texttt{Salesforce/wikitext}\footnote{\url{https://huggingface.co/datasets/Salesforce/wikitext}}. 
The WebText and WritingPrompts datasets are accessed through the data interface released in the Github repository \texttt{bloomberg/MixCE-acl2023}\footnote{\url{https://github.com/bloomberg/MixCE-acl2023}}. 
The source-code corpus CodeParrot is downloaded from the Hugging Face dataset hub
\path{codeparrot/codeparrot-clean}\footnote{\url{https://huggingface.co/datasets/codeparrot/codeparrot-clean}}, 
while the multilingual evaluation dataset GermanQuAD is obtained from from the Hugging Face dataset repository \texttt{deepset/germanquad}\footnote{\url{https://huggingface.co/datasets/deepset/germanquad}}. 

\textbf{Setup of methods}. 
In the implementation of our proposed continuous perturbation framework, the trained base language model is frozen and only the low-rank adaptation (LoRA) parameters are optimized, which substantially improves training stability and reduces memory consumption. 
LoRA adapters are injected into the attention projection layers of the base model, with LoRA rank $8$, scaling factor $16$, and dropout rate $0.05$. 
The perturbation generator $T_{\beta}$ adopts the same neural-network architecture as that used in the simulation studies, except that the latent perturbation dimension is increased to $r=64$ in the real-data experiments.
For fairness, in the discrete perturbation baseline, we do not merge the perturbed samples with the original training data and instead train solely on the perturbed dataset. 
For both the discrete method and NEFTune, the perturbation intensity is set to $0.0125$, which corresponds to the effective perturbation level induced by the value $0.025$ used in the merged-data real-world experiments of \cite{Cenetal2026}.

\textbf{Training details}. 
We use the AdamW optimizer \citep{loshchilov2019decoupled} together with a linear learning-rate scheduler, and set the gradient accumulation step to $1$. 
Since our proposed method updates only the LoRA parameters of the base language model, whereas the baseline methods optimize all model parameters, we adopt different learning rates for different methods; the detained settings are reported in Table~\ref{tab:lr_set}. 
For the proposed continuous perturbation framework, we set the number of perturbation samples to $K=5$. 
Throughout all experiments, We use a batch size of $16$ and set the maximum input length to $64$ tokens. 

\begin{table}[!ht]
\centering
\caption{Learning-rate configurations for different base language models. Here, $\mathrm{lr}_{P_{\theta}}$ denotes the learning rate used for optimizing the LoRA parameters in the proposed method or all trainable model parameters in the baseline methods, while $\mathrm{lr}_{T_{\beta}}$ denotes the learning rate used for training the perturbation neural network in the proposed method.}
\small
\setlength{\tabcolsep}{8pt}
\renewcommand{\arraystretch}{1.15}

\begin{tabular}{l|cc|cc|cc}
\toprule

\multirow{2}{*}{Method}
& \multicolumn{2}{c|}{OPT}
& \multicolumn{2}{c|}{Qwen3}
& \multicolumn{2}{c}{GPT-Neo} \\

\cline{2-7}

& $\mathrm{lr}_{P_{\theta}}$
& $\mathrm{lr}_{T_{\beta}}$
& $\mathrm{lr}_{P_{\theta}}$
& $\mathrm{lr}_{T_{\beta}}$
& $\mathrm{lr}_{P_{\theta}}$
& $\mathrm{lr}_{T_{\beta}}$ \\

\midrule

Ours
& $5\times10^{-7}$
& $5\times10^{-8}$
& $1\times10^{-7}$
& $1\times10^{-8}$
& $5\times10^{-8}$
& $5\times10^{-9}$ \\

Baselines
& $5\times10^{-5}$
& --
& $5\times10^{-6}$
& --
& $2\times10^{-6}$
& -- \\

\bottomrule
\end{tabular}

\label{tab:lr_set}
\end{table}

\textbf{Evaluation details}. 
To more reliably evaluate the learned model distribution \(P_\theta\), we follow \cite{eikema2020map} and \cite{ren2024emo} by using unbiased ancestral sampling as the default decoding strategy in all experiments. 
To ensure a consistent comparison across different datasets, we adopt the same evaluation protocol throughout all experiments. 
Specifically, each generation is conditioned on a $20$-token prefix, followed by the generation of an $80$-token continuation. 
For every prefix, we repeat the sampling procedure over $5$ independent runs and report the averaged PPL, Mauve, and ROUGE-1 scores.

\end{document}